\documentclass[11pt]{article}
\usepackage{format}
\usepackage{
 macros}
\usepackage{paralist}

\newcommand{\GoodV}{\G(\D_V)}

\newcommand{\Good}{\G(\D)}

\title{Publicly-Verifiable Certificates for Statistical Algorithms}
\author{Michael Ngo\thanks{Research completed while at Cornell University, supported by the Bowers Undergraduate Research Experience (BURE) and the Dean Archer Undergraduate Research Program.
This research was supported by a gift from Google.}\\MIT\\\texttt{mingo@mit.edu} \and Michael P. Kim\thanks{This research was supported by a gift from Google.}\\Cornell University\\\texttt{mpk@cs.cornell.edu}}
\date{1 April 2026}

\begin{document}
\maketitle

\begin{abstract}

Following Goldwasser, Rothblum, Shafer, and Yehudayoff, who defined a framework for interactive proofs of learning \cite{goldwasser2021interactive}, we initiate the study of non-interactive proofs of learning.
We define and study a new notion: \emph{Publicly-Verifiable Certificates of Statistical Validity} (pvCSVs), which allow for public, distributionally-robust certification that the result of a learning algorithm is valid.
In a pvCSV, a learner publishes a hypothesis $h$ and corresponding certificate $\pi$; then, \emph{any} user, who holds a user-specific distribution, can read the pair $(h,\pi)$ and determine efficiently whether the hypothesis is valid \emph{according to the user-specific distribution}.

We construct pvCSVs in the context of Adaptive Statistical Query (SQ) Algorithms.
To certify SQ algorithms that makes $k$ adaptive queries, we construct pvCSVs where the sample complexity scales with $O(\log k)$, whereas the sample complexity of the best learning algorithms scale with $\tilde{O}(\sqrt{k})$.
More generally,
we study proof systems for learning in the SQ model, demonstrating the model's strengths as well as its limitations.

 \end{abstract}

\section{Introduction}
\label{sec:intro}

Training large-scale AI models using statistical machine learning is notoriously costly.
Due to the resource demands of running ML algorithms, AI users rely upon pre-trained models from a handful of tech companies.
These companies hold enough data---so they claim---to train general-purpose models that are effective across a wide range of settings.
In this setup, however, users receive no guarantee that the AI models were trained appropriately.
If users fear that the training data does not adequately reflect their setting,
they
must investigate, for themselves, whether the model produces errors (or worse, harms) within their application.

Inspired by these issues, Goldwasser, Shafer, Rothblum, and Yehudayoff \cite{goldwasser2021interactive} introduced and studied the problem of delegating machine learning, through the PAC Verification framework.
Building on the classic model of interactive proofs \cite{gmr,babai1985trading}, PAC Verification formalizes the problem:
a statistically-limited user (the \emph{verifier}) interacts with a powerful, but untrusted learner (the \emph{prover}), who aims to convince the user that a given model (the \emph{hypothesis}) is valid.
As in cryptographic proof systems, PAC Verification requires that protocols for delegating learning satisfy formal notions of \emph{completeness} and \emph{soundness}.
To date, results in the area focus on developing protocols for verifying Agnostic PAC Learning \cite{Valiant1984,Haussler1992,Kearns1994} for specific concept classes \cite{goldwasser2021interactive,mutreja2023pac,gur2024power}.
As a notable exception, Mutreja and Shafer \cite{mutreja2023pac} also introduce a notion of delegation of Statistical Query (SQ) algorithms for loss minimization.

An essential element of prior works on delegation of learning is \emph{interaction}.
To establish whether a given hypothesis is valid, the prover and verifier exchange a sequence of messages in an online fashion, after which the verifier chooses to accept or reject the prover's hypothesis.
As a concrete example, the original work on PAC verification shows how to delegate the Goldreich-Levin algorithm \cite{Goldreich1989} (which requires point query access to the unknown function), when the verifier only has i.i.d.\ labeled samples.
In this proof system, the verifier uses its interactions with the prover to label the point queries, while cleverly hiding some points whose labels are known to the verifier to ensure soundness.
The early works on PAC verification have demonstrated that interactive proof systems provide a powerful tool for efficiently checking the results of an expensive ML computation.\footnote{In these works, and in our paper, ``efficiency'' primarily focuses on statistical resources, rather than computation.}

Interaction, however, also presents challenges.
Most immediately, running an interactive proof requires the verifier and prover to be online at the same time to execute the protocol.
Additionally, each execution of the interactive proof may require the prover to answer execution-specific challenges, including re-running the original ML computation.
Given the immense cost of training ML models in the first place, providers may be unwilling to participate in the interactive proof more than once (if at all).
In such a setting, where the interactive proof is executed once---between the learner and a single verifier---many users would have to place their trust in a single entity.
Even if users agree that the verifier is generally trustworthy, as before, individual users may worry that the verifier's data does not represent their setting and applications.

\paragraph{This Work.}
We initiate the study of non-interactive proofs of learning.
Our study leads us to a new notion:
\emph{Publicly-Verifiable Certificates of Statistical Validity (pvCSVs)}.
pvCSVs allow for public, distributionally-robust certification that the result of a learning algorithm is valid.
In particular, a pvCSV allows a learner to publish a hypothesis $h$ and certificate of validity $\pi$ that allows \emph{any} downstream user to subsequently verify that the hypothesis is statistically valid on a \emph{user-specified} distribution.

We can understand the semantics of a pvCSV by imagining two worlds.
\begin{itemize}
    \item In the first world, a user collects a huge amount of data from a distribution $\D$, relevant to their setting and application.
    Then, they correctly execute a statistical learning algorithm $\A$ on top of this data to obtain a hypothesis $h_\mathrm{ideal}$.
    \item In the second world, a centralized, well-resourced learner publishes a pvCSV $(h_\mathrm{real},\pi)$ for the algorithm $\A$; the same user from before collects a much more modest amount of data from $\D$, then reads and verifies $(h_\mathrm{real},\pi)$, using the small amount of user-specific data.
\end{itemize}
A pvCSV guarantees that if $(h_\mathrm{real},\pi)$ passes verification, the two worlds produced equally-valid hypotheses $h_\mathrm{ideal} \approx_{\D,\A} h_\mathrm{real}$ according to the user-specific data distribution and the learning algorithm $\A$.
This guarantee must hold even though the user has no knowledge of the data distribution used to produce the pvCSV.
As such, pvCSVs resolve the key shortcomings of interactive proofs for delegation of learning: the learner can execute the expensive training algorithm once and (with little overhead) generate a corresponding pvCSV certificate that any user can subsequently verify.

\paragraph{Certification of Adaptive Data Analysis.}
We make our study of pvCSVs concrete by revisiting the question of Adaptive Data Analysis \cite{dwork2015preserving} in the Statistical Query Model \cite{kearns1998efficient}.
Many tools for learning from data---including workhorse ML algorithms like gradient descent---can be framed as adaptive statistical algorithms.
In such an algorithm, the learner is allowed to ask a sequence of queries of the data distribution (e.g., \emph{What is the gradient of the expected loss over $\D$?}), where each query may depend on the results from prior queries.

Formally, we consider learning algorithms $\A$ that interact with a Statistical Query (SQ) oracle $\O$: given a tolerance $\tau$ and a query $q$, $\O(q)$ responds with a $\tau$-accurate estimate of the expectation of the predicate $q$ over the data distribution.
Critical to our investigation, the algorithm may select its sequence of queries \emph{adaptively} based on the prior responses.
That is, the algorithm's choice of the $i$-th query $q_i$ may depend arbitrarily on prior queries $q_1,\hdots,q_{i-1}$ and responses $\O(q_1),\hdots,\O(q_{i-1})$, (e.g., \emph{What is the gradient at the $i$-th iterate, after taking $i-1$ gradient descent steps?}).

While the paradigm of adaptive data analysis is a versatile and powerful tool for learning, such algorithms are known to be statistically expensive.
About a decade ago, \cite{dwork2015preserving} identified adaptivity as a key issue in statistical algorithms.
To maintain the statistical validity of an adaptive analysis requires the learner to either resample fresh data to support every new query, or to employ sophisticated (differentially-private) algorithms for answering queries in a way to prevent overfitting to the data set \cite{dwork2015preserving,dwork2015generalization,dwork2015reusable,bassily2016algorithmic,feldman2018calibrating,jung2019new,dagan2022bounded,blanc2025subsampling}.
To answer $k$ adaptively selected statistical queries, the best algorithms use a number of samples scaling (roughly) with $\sqrt{k}$, and in fact, this dependence is essentially tight \cite{hardt2014preventing,steinke2015interactive}.

In other words, no matter what techniques are employed, adaptive statistical algorithms require \emph{exponentially} more data than non-adaptive (batch) statistical analyses of similar size.
In this work, we ask when we can certify the results of adaptive data analysis more efficiently than learning.

\subsection{Our Contributions}

We develop proof systems for delegating \emph{arbitrarily-adaptive} statistical algorithms
where the verifier requires samples scaling only with the \emph{non-adaptive} complexity.
In doing so, we realize an \emph{exponential} gap, between the sample complexity required for executing SQ algorithms versus verifying them.
Moving beyond prior work on \emph{interactive} proofs for learning, we build novel \emph{non-interactive} proof systems---namely, \emph{publicly-verifiable Certificates of Statistical Validity}---that enable a new form of distributionally-robust validation of statistical learning.
Along the way, we develop a number of extensions to earlier models of proof systems for delegation of learning.

\paragraph{Publicly-Verifiable Certificates of Statistical Validity.}
In Section~\ref{sec:pvcsv}, we introduce our primary contribution:  a new notion of proof that allows for public, distributionally-robust certification of learning.
A \emph{publicly-verifiable Certificate of Statistical Validity} (pvCSV) is a non-interactive proof system that allows any verifier to certify the results of a statistical computation \emph{with respect to the verifier's own distribution}.
A pvCSV allows a single, well-resourced learner---the prover---to publish the results of a statistical algorithm in a way that can be checked efficiently (with less resources) by any downstream verifier; in particular, the verifiers need not hold the same distribution as the prover.
Instead, the proof system guarantees a \emph{universal soundness} property such that if the verifier accepts the proof, then the results of the statistical algorithm are valid on the verifier's distribution---even if the algorithm was executed using samples from a different distribution.
\begin{newdef}[pvCSV, informal]
    A \define{publicly-verifiable Certificate of Statistical Validity} is a non-interactive proof system, where a prover $P$ with distribution $\D_P$ publishes a hypothesis $h$ paired with certificate $\pi$.
    Any verifier $V$ with distribution $\D_V$ can read the pair $(h,\pi)$ and accept or reject where the following guarantees hold with high probability.
    \begin{compactitem}
        \item \define{Completeness}: if $\D_P = \D_V$, there exists an honest prover pair $(h,\pi)$ such that $h$ is valid for $\D_V$ and $V$ accepts.
        \item \define{Universal Soundness}: for any verifier $V$ with distribution $\D_V$, for any (possibly-cheating) prover pair $(\tilde{h},\tilde{\pi})$, if $V$ accepts, then $\tilde{h}$ is actually valid for $\D_V$.
    \end{compactitem}
\end{newdef}
One way to understand the guarantee of pvCSVs is as a robust proof of statistical validity, \emph{without an explicit assumption about distributional shifts}.
Rather than positing some known relationship between the prover and verifier distributions, the verification procedure works for any $\D_P$ and $\D_V$ and results in acceptance whenever the published certificate---derived from the execution of a learning algorithm $\A$ using the prover's distribution---reflects some legitimate execution on the verifier's distribution.
While our notion of completeness assumes that $\D_P = \D_V$, the guarantee is more subtle than this equality would suggest.
The distributions $\D_P$ and $\D_V$ may differ significantly in composition, but if the verifier accepts, then (by universal soundness) the hypothesis $h$ is valid for $\D_V$ (because the distributions are  indistinguishable according to some invocation of $\A$).
The careful reader will note that this ``universal soundness'' condition is actually implied by standard soundness for delegation of learning,\footnote{The honest prover using $\D_P$ may be viewed as a cheating prover with respect to the verifier holding $\D_V$.} but holds new significance in the context of non-interactive proofs.
Our view on soundness, paired with a non-interactive proof system, allows us to realize certificates of learning that can be verified publicly by any user.

We can use various measures of complexity to evaluate the quality of a pvCSV construction.
Principle to our work is sample complexity: we aim for pvCSVs where the verifier saves considerably in the number of samples needed from $\D_V$ compared to the number of samples required for learning (or proving).
Further, we can track other measures like time complexity (of both the verifier and the honest prover) and proof length as well.
We define pvCSVs formally in Section~\ref{sec:pvcsv} and provide a thorough discussion of the notion and its properties (like validity and universal soundness) therein.

With this key definition in place, the main technical contributions of this work are to construct pvCSVs for adaptive statistical algorithms within the SQ learning framework.
Our pvCSVs achieve an \emph{exponential} gap in the sample complexity between SQ verification and SQ learning.
While the core idea of each of our constructions is similar, the resulting protocols differ in important ways based on properties of the original SQ algorithm.
As we show, the way that the algorithm uses randomness and the way the SQ oracle is allowed to adapt to the algorithm's internal state are key factors in building universally-sound pvCSVs.

\paragraph{pvCSVs for Deterministic SQ Algorithms.}
In Section~\ref{sec:deterministic}, we consider SQ algorithms that do not use internal randomness.
These ``deterministic'' SQ algorithms are considerably simpler to reason about, and as such, we obtain strong, statistically-sound pvCSVs for all such algorithms.\footnote{We note that such ``deterministic'' algorithms may still have non-deterministic behavior, due to the algorithm's adaptivity to the responses from the SQ oracle, which is assumed to be $\tau$-accurate, but otherwise, adversarial.}

\begin{result}
    \label{result:deterministic}
    Let $\A$ be a deterministic SQ algorithm that learns a concept $\G$.
    Suppose $\A$ makes $k$ adaptive queries to a $\tau$-accurate SQ oracle $\O$.
    There exists a pvCSV scheme for certifying $\G$ (with constant failure probability) achievable in the following complexities.
    \begin{compactitem}
        \item Honest prover sample complexity:  $m_P \le \tilde{O}(\sqrt{k}/\tau^2)$
        \item Verifier sample complexity:  $m_V \le O(\log(k)/\tau^2)$
        \item Certificate size:  $\card{\pi} \le O(k \cdot \log(1/\tau))$
    \end{compactitem}    
\end{result}
In other words, even though the best learning algorithms use $\tilde{O}(\sqrt{k})$ samples to support $k$ adaptively-chosen queries, there is a short certificate (namely, the results of the queries) that convinces a verifier with considerably less information about the distribution in $O(\log(k))$ samples.
The bounds of Theorem~\ref{result:deterministic} follow from the prover and verifier collectively executing a direct simulation of the original SQ algorithm with no overhead for the honest prover; in this sense, the delegation scheme can be thought of as \emph{doubly-efficient} in its statistical complexity, \`a la \cite{gkr}.
Computationally, both the prover and the verifier need to simulate the execution of the underlying SQ algorithm, so the time-complexity scales roughly linearly in the running time of the original algorithm.

\paragraph{pvCSVs for Randomized SQ Algorithms and SQ Protocols.}
In Section~\ref{sec:randomized}, we investigate how to build pvCSVs for SQ algorithms that use randomness.
Randomness introduces significant technicality in the resulting proof systems.
Here, we give an informal description of our results, including various definitions we develop to obtain the results.

Towards a pvCSV for randomized SQ algorithms, we need to reason about how the algorithm---and the (possibly-adversarial) SQ oracle---may act adaptively to the randomness.
First, we say that a randomized SQ algorithm operates in $\ell$ ``epochs'' if it can be broken down into a series of $\ell$ deterministic SQ algorithms, which each take a fresh random string as an input.
Next, we say that an SQ oracle is ``public-state'' if it may choose its query responses as a function of the SQ algorithm's internal randomness (in contrast to an ``oblivious'' oracle, which has no knowledge of the algorithm's randomness).
We construct pvCSVs in the Random Oracle Model (ROM) \cite{bellare1993random} for all constant-epoch randomized SQ algorithms that are correct using a public-state SQ oracle.
\begin{result}\label{result:randomized}
    Let $\A$ be a randomized SQ algorithm that learns a goal $\G$ with high probability.
    Suppose $\A$ makes $k$ adaptive queries to a $\tau$-accurate public-state SQ oracle $\O$, over an execution of $O(1)$ epochs.
    There exists a computationally-sound pvCSV scheme in the ROM for certifying $\G$ (with constant failure probability) achievable in the following complexities.
    \begin{compactitem}
        \item Honest prover sample complexity: $m_P \le \tilde{O}(\sqrt{k}/\tau^2)$
        \item Verifier sample complexity: $m_V \le O(\log(k)/\tau^2)$
        \item Certificate size: $\pi \le O(k \cdot \log(1/\tau))$
    \end{compactitem}   
\end{result}

In fact, Theorem~\ref{result:randomized} follows from a much more general result.
We start by adapting the definition of delegation of learning \cite{goldwasser2021interactive} to the statistical query setting, yielding a format for interactive learning which we call \emph{SQ protocols}.
We show that a large class of interactive SQ protocols (of which randomized SQ algorithms are a special case) can be compiled into a \emph{canonical SQ protocol} with desirable properties.
The class consists of SQ protocols where the verifier may send random and non-random challenges to the prover and may make private statistical queries (not sent to the prover) to its public-state oracle; we call these protocols ``mixed-message, private-query'' SQ protocols.
After compilation, we obtain a canonical SQ protocol with the same completeness and soundness guarantees as the original, but which is public-coin (verifier only sends random challenges) and public-query (verifier reveals all SQs to the prover).
Importantly, the canonical verifier still only makes a single batch of statistical queries.
\begin{resultlem}
    \label{result:canonical}
    Suppose $(P,V)$ is a mixed-message, private-query SQ protocol where the verifier makes $k$ adaptive queries to a $\tau$-accurate public-state SQ oracle and verifies concept $\G$.
    There exists a canonical public-coin, public-query SQ protocol $(\mathsf{P}_\mathsf{can},\mathsf{V}_\mathsf{can})$ that verifies $\G$ with the same completeness/soundness as $(P,V)$, with the following properties.
    \begin{compactitem}
        \item $\mathsf{P}_\mathsf{can}$ makes at most $k$ additional statistical queries compared to $P$;
        \item $\mathsf{V}_\mathsf{can}$ makes a single, non-adaptive batch of $k$ statistical queries.
    \end{compactitem}
\end{resultlem}
This canonical SQ protocol can then be compiled into a pvCSV via a Fiat-Shamir transformation \cite{fiat1986prove}, whose blow-up in soundness scales exponentially with the ``epoch complexity'' of the original protocol.
The sample complexities claimed in Theorem~\ref{result:randomized} follow by giving sample-based implementations of an adaptive SQ oracle (for the prover) and non-adaptive SQ oracle (for the verifier).
As a consequence, we obtain computationally-sound pvCSVs for a much broader class of learning goals---those that can be solved by an interactive SQ protocol according to Lemma~\ref{result:canonical}.

As may be evident, to appropriately reason about delegation of learning in this context, we need to reason about a number of novel concepts (e.g., SQ protocols, public-coin, public-query, epoch complexity, etc.).
An additional key contribution of our work is laying out precise definitions for these notions that arise in the study of proof systems for learning, presented formally in Section~\ref{sec:random-prelims}.
We discuss our construction of pvCSVs for randomized algoirthms/SQ protocols in greater detail within the Technical Overview.

\paragraph{Beyond pvCSVs: The Limits of SQ Protocols.}

Building pvCSVs, particularly for randomized SQ algorithms, required investigating aspects of more general interactive protocols for statistical query learning.
We complement our constructions of pvCSVs with a few results about the strengths and limitations of SQ protocols.

We show that, quite generically, verification of SQ protocols can be made statistically non-adaptive, in the sense that the verifier issues a single batch of statistical queries, and thus, has sample complexity scaling logarithmically in the number of queries.
The exact class of protocols which we can delegate soundly in this manner is a bit technical, but it consists of a large class of ``public-query'' SQ protocols.
In particular, the class of protocols includes private-coin protocols, where the verifier maintains secret randomness from the prover, but allows the prover to know which statistical queries the verifier issues.

\begin{resultprop}[Informal]
\label{result:protocol}
For every
public-query SQ protocol where the verifier makes $k$ adaptive queries to a $\tau$-accurate public-state SQ oracle, there is an equivalent SQ protocol where the verifier $V$ makes a single, non-adaptive batch of queries; that is, $V$ has sample complexity $m_V \le O(\log(k)/\tau^2)$.
\end{resultprop}
This result shows that statistical validation of a large class of SQ protocols can be done in non-adaptive sample complexity.
A natural question, then, is whether this statistically-efficient verification scheme can be made computationally-efficient.
Unfortunately, we show a barrier to generic computational savings in the SQ model.
Piggybacking off of a sample complexity lower bound given by \cite{mutreja2023pac} for PAC Verfication, we obtain a lower bound on the verifier's \emph{query complexity} for the same class of SQ protocols from Proposition~\ref{result:protocol}.
\begin{resultcor}[Informal]
    \label{result:lb}
    For every hypothesis class $\H$ of VC dimension $d$, in any SQ protocol (as in Proposition~\ref{result:protocol}) that $\eps$-PAC Verifies $\H$ using a $\tau$-accurate SQ oracle for $\tau \approx \eps$, the verifier makes $2^{\Omega(\sqrt{d})}$ statistical queries.
\end{resultcor}

\paragraph{SQ Verification under Differential Privacy.}
Finally, we show that the verifiers of all of our pvCSVs and interactive SQ protocols can be implemented under Differential Privacy \cite{dwork2006calibrating}.
Differential Privacy (DP) is the de facto notion to protect individuals' data in statistical analyses, but it can be statistically costly: implementing SQ algorithms under DP requires sample complexity akin to adaptive data analysis.
This result shows that a verifier can validate the results of a statistical analysis much more efficiently, while still maintaining privacy over their own data set.

The new private verifiers use essentially the same number of samples as the non-private verifiers.
Concretely, we state the result for the pvCSV verifiers.
\begin{resultprop}[Informal]
    \label{result:dp}
Consider the pvCSV verifier from either Theorem~\ref{result:deterministic} or Theorem~\ref{result:randomized}.
    There exists an $\eps$-DP implementation of the verifier over the samples from $\D_V$
(with constant failure probability) whose sample complexity scales as
    $O(\log(k)/\tau^2 + \log(k)/\tau \eps)$.
\end{resultprop}

\subsection{Technical Overview and Discussion of Results}
\label{sec:overview}

In the remainder of the introduction, we give a more detailed overview of our models and results.
Throughout, we aim to provide pointers into the main text for formal presentation.
We include discussion of the significance of the results, connections to prior works, as well as possible extensions.

The primary goal of our work is to take an arbitrarily-adaptive SQ algorithm $\A$ and turn it into a pvCSV that can be verified efficiently.
At a high-level our approach is simple:  require the prover to run a direct simulation of $\A$; then, check the prover's work.
In particular, the statistically-expensive aspect of adaptive data analysis is generating the sequence of queries $q_1,\hdots,q_k$.
Once the sequence has been generated, however, the answers to the queries $q_1,\hdots,q_k$ can be checked in a single, non-adaptive batch of $k$ statistical queries.

This observation immediately suggests a non-interactive proof system for delegating ``deterministic'' SQ algorithms that do not use any internal randomness, which is the focus of Section~\ref{sec:deterministic}.
\begin{compactitem}
    \item The honest prover $P$ executes $\A$ using their own SQ oracle $\O_P$ to answer any necessary queries.
    Along the way, the prover records the queries and results, and at the end, when $\A$ outputs some hypothesis $h$, $P$ sends $h$ and $\pi = \langle q_1, \O_P(q_1),\hdots,q_k,\O_P(q_k) \rangle$ to the verifier $V$.
    \item To certify that $h$ is valid, the verifier must re-run $\A$, but rather than using its own oracle, it will answer any statistical queries using the answers from $\pi$.
    If at any point, the verifier's execution $\A$ requires a query $q$ that is not provided in $\pi$, then $V$ rejects immediately, since the transcript sent by the prover was not consistent with the execution of $\A$.
    If the transcript is consistent, then the verifier concludes the protocol by issuing the batch of statistical queries from $\pi$ to its own oracle $\O_V$, and checks that $\O_V(q_i)$ is sufficiently close to the reported $\O_P(q_i)$ for each query.
\end{compactitem}
Note that even though we think of $\A$ as deterministic since it does not use randomness, its adaptivity to the SQ oracle responses introduces non-determinism, so verifying the consistency of the transcript is a non-trivial aspect of the verifier's check.

Given the simplicity of this delegation scheme, both the honest prover and verifier can be implemented very efficiently.
The only overhead of the honest prover is to record the results of their statistical queries to be sent to the verifier (or posted for public verification), so the sample and time complexities scale precisely with the original complexities of the SQ algorithm $\A$.
To answer a sequence of adaptive queries to $\tau$-accuracy, the prover can be implemented in $\tilde{O}(\sqrt{k}/\tau^2)$ samples \cite{bassily2016algorithmic,dagan2022bounded, blanc2025subsampling}.
The verifier also has to execute the algorithm $\A$, so there is no computational savings, but the statistical savings are exponential.
By concentration bounds, checking a batch of $k$ statistical queries to $\tau$-accuracy can be achieved from $O(\log(k)/\tau^2)$ samples.

\paragraph{Understanding Universal Soundness.}
One of the key selling points of pvCSVs is the universal soundness condition: that any verifier holding $\D_V$ which may differ significantly from the prover's distribution $\D_P$ can check the certificate while maintaining soundness.
As described above, our pvCSV construction does not explicitly distinguish between the prover's distribution $\D_P$ and the verifier's $\D_V$.
But in a sense that can be made formal, to obtain universal soundness, we only need for standard soundness to hold from the verifier's perspective.
Specifically, even if the pvCSV $(h,\pi)$ was generated honestly by a prover $P$ holding $\D_P$, to the verifer holding $\D_V$, we can imagine $P$ to be a potential cheating prover.
In this case, the verifier may reject outright, or may accept if $h$ actually satisfies the learning goal over $\D_V$.
The verifier's final statistical validation---paired with the correctness guarantee of SQ algorithms---ensures that if the reported query responses from $\pi$ are sufficiently close to the expectations on the verifier's distribution $\D_V$, then $h$ is the result of some valid invocation of the SQ algorithm $\A$ over $\D_V$.

\paragraph{Handling Randomness.}
For algorithms that use randomness, we cannot simply trust the prover to report a direct simulation of the execution of $\A$, using untrusted randomness.
Instead, a natural idea for a pvCSV is to convert the randomized algorithm into a public-coin protocol, then apply a Fiat-Shamir transformation \cite{fiat1986prove}, to obtain a non-interactive proof.
In Section~\ref{sec:randomized}, we show that this approach works to yield pvCSVs, but we need to be careful in how we reason about the algorithm's use of randomness, as well as how the SQ oracle affects the proof of soundness.

Starting from a randomized SQ algorithm $\A$, we imagine breaking the algorithm into a series of ``epochs'' where in the $i$-th epoch, the algorithm samples fresh randomness $r_i$, and then executes a deterministic SQ algorithm $\A_i$ until the end of the epoch.
Every randomized algorithm can be broken into epochs, but different algorithms require more or fewer epochs.
At the low extreme, an algorithm that uses a random initialization, then executes deterministically, would have epoch complexity $1$; at the other extreme, a stochastic optimization algorithm that makes random choices at every iteration based on fresh independent coins will have high epoch complexity.
The epoch complexity of the algorithm $\A$ controls the complexity of the resulting proof system.

In particular, we can turn any SQ algorithm of epoch complexity $\ell$ into a public-coin interactive proof in the SQ model of round complexity $\ell$.
At the start of each epoch $i$, the verifier $V$ sends randomness $r_i$ to the prover $P$.
Then, the honest prover $P$ simulates the deterministic algorithm $\A_i(r_i)$ using its own oracle $\O_P$ to answer any statistical queries.
At the end of the epoch, the prover can return a transcript, similar to in the deterministic case, that summarizes the statistical queries and oracle responses.
At the end of all epochs, the verifier can issue a non-adaptive batch of queries to $\O_V$ to ensure that the responses were all sufficiently accurate.

As described, the SQ protocol is complete, but has a subtle issue with soundness without further assumptions.
In the original algorithm, $\A$ issues its queries to an SQ oracle $\O$, whereas in the protocol, the verifier $V$ simulates $\A$ but delegates the SQs to the prover.
The key distinction between these two setups is that the prover also receives the internal randomness of the algorithm $r_i$, as the public-coin message at the start of each epoch.
In other words, to ensure soundness, the original randomized algorithm $\A$ must be correct even for SQ oracles that have full knowledge of the state of $\A$.
Algorithms whose correctness hinges on the obliviousness of $\O$ to the randomness of $\A$ cannot be delegated in this way while maintaining soundness.
Formally, we define ``public-state'' SQ oracles in Section~\ref{sec:random-prelims}, along with the corresponding correctness notion for SQ algorithms.

Once we restrict our delegation to randomized SQ algorithms which are correct under this more powerful SQ oracle, then we obtain a sound public-coin protocol.
To obtain pvCSVs, we appeal to the Random Oracle Model (ROM), and show how to apply the Fiat-Shamir transform to our protocol.
There is some subtlety in defining the notion of universal soundness for pvCSVs in the ROM and then subsequently arguing that Fiat-Shamir applied to our protocol obtains such soundness.
With the appropriate definitions in place, the soundness proof follows by following the state restoration approach of \cite{ben2016interactive,chiesa2024snargsbook}.

\paragraph{Beyond SQ Algorithms.}
Building pvCSVs for randomized SQ algorithms required us to define and investigate a number of more general models of proof systems in the SQ model.
A key result we show in Section~\ref{sec:canonical} is that a diverse class of interactive SQ protocols can be compiled into a canonical public-coin SQ protocol.
In a sense, this result is an analogue of the established understanding of public-coin protocols for delegation of computation.
We show that, quite generically, if an SQ protocol consists of a verifier that reveals its randomness to the prover, there is a canonical version of the protocol where the only messages the verifier sends to the prover are its randomness.
Further, the verifier need not maintain any private statistical queries, but can delegate all of its queries to the prover, then execute one final non-adaptive validation.
This result holds in the same model of public-state SQ oracle, which may adapt to the state of the verifier.
With this canonical compiler for SQ protocols, we can similarly obtain computationally-sound pvCSVs in the ROM, by applying Fiat-Shamir to the canonical protocol.

\paragraph{Towards Computationally-Efficient Certification of Learning.}
In this work, our focus is on developing statistically-efficient proofs of learning.
We make no effort to optimize the verifier's running time.
Our pvCSV verifier for an SQ algorithm $\A$ runs in time proportional to the running time of $\A$.
A natural question in the study of proof systems is whether we can save on computation time during verification.

In Section~\ref{sec:limits}, we show some barrier to generic speed-ups (at least within the SQ framework).
In particular, Corollary~\ref{result:lb} shows that there are statistically learnable VC classes, for which there is no computationally-efficient SQ verification scheme, even if we allow for interaction.
Our SQ lower bound is actually a consequence of the \emph{efficiency} of our verification schemes (in terms of the number of statistical queries) and an existing lower bound on the samples required for verification from \cite{mutreja2023pac} in terms of the VC dimension.
In other words, a fast SQ verification scheme (which doesn't make many SQ queries) would imply an impossibly-statistically-efficient verification scheme for the VC class studied by \cite{mutreja2023pac}.

Nevertheless, in some sense, our exploration of pvCSVs and SQ protocols provides new design principles for simultaneously statistically and computationally efficient verification.
Recall that the canonical protocol of Lemma~\ref{result:canonical}, which compiles into a pvCSV as in Theorem~\ref{result:randomized}, applies not just to randomized algorithms but to a broad class of SQ protocols.
Further, the verifier in the resulting pvCSV runs in time proportional to (or less than) the verifier in the original protocol---not the full simulation of $(P,V)$.
In other words, if we can design SQ protocols that reduce verification time, then we may get simultaneous statistical and computational efficiency for free (provided the protocol is covered by Lemma~\ref{result:canonical}).

Concretely, our results show that minimizing the epoch complexity of a learning procedure---even if that entails designing an interactive proof system---may actually lead to more efficient pvCSVs than simply focusing on delegating algorithms for learning.
Such observations may also motivate a deeper study of pseudorandomness in machine learning; for instance, if a randomized SQ algorithm can actually be proved correct under a weaker source of randomness, it may reduce the epoch complexity and improve efficiency.

\paragraph{Publicly-Verifiable End-to-End DP.}
Finally, in Section~\ref{sec:dp-verification}, we show Proposition~\ref{result:dp}, which implies a mechanism for verifying pvCSVs statistically efficiently under Differential Privacy with respect to the verifier's samples.
Given such a verifier, it's tempting to wonder whether we can achieve End-to-End DP over both the prover's samples from $\D_P$ and the verifier's from $\D_V$.
The existence of a DP verifier implies an honest prover strategy to achieve such a goal.
Due to the connections between adaptive data analysis and DP, this honest prover strategy is no more expensive than accounting for adaptivity.

That said, in our setting, we'd have no guarantees about the behavior of the cheating prover, who might violate DP arbitrarily.
To handle the possibility of a cheating prover, we could employ a scheme for Certified DP of \cite{bell2024certifying}.
Certified DP ensures that a prover's release of statistical queries follows a DP mechanism with respect to some committed-to, but untrusted database $X \sim \D_P$, in a manner that can be verified publicly.
In this sense, adding our DP statistical validation scheme as post-processing to the output of a Certified DP mechanism would allow the verifier to be convinced of End-to-End DP, while also ensuring statistical accuracy of queries from the untrusted database $X$ under the verifier's distribution $\D_V$.

\subsection{Related Works}

Our study of pvCSVs and SQ protocols lives in the intersection of learning theory and cryptographic proof systems.
We highlight some of the most relevant related lines of work.

\paragraph{Proof Systems for Learning and Statistics.}
Proof systems for statistical algorithms grew out of the literature on interactive (cryptographic) proof systems \cite{gmr,babai1985trading}.
The PAC Verification, introduced in \cite{goldwasser2021interactive}, led to a number of results about strengths and limitations of interactive proofs of learning \cite{mutreja2023pac,gur2024power}.
Most related to our work is the work of Mutreja and Shafer \cite{mutreja2023pac}, who also thought about the issue of delegating SQ algorithms in the context of loss minimization.
In fact, in a subsequent journal version \cite{MSjournal}, the authors independently include a result analogous to Proposition~\ref{result:protocol}, on the interactive delegation of adaptive SQ algorithms.

Other works have investigated certifying properties of learning algorithms.
\cite{bell2024certifying} recently designed mechanisms for public cerification of differential privacy (and other private probabilistic mechanisms) in the release of statistical queries.
Following this work, \cite{bell2025efficient} extended the techniques to give a publicly-verifiable implementation of DP Stochastic Gradient Descent, without appealing to heavy-handed cryptographic primitives for certifying generic computations.

Beyond proofs for learning algorithms, there has been significant work investigating proof systems for properties of distributions.
Initiated by \cite{chiesa2018proofs}, a sequence of works \cite{herman2022verifying,herman2023doubley,herman2024interactive,herman2024verify} has established doubly-efficient proof systems for distribution testing.

\paragraph{Adaptive Data Analysis.}
Since its identification as a key algorithmic challenge \cite{dwork2015reusable,dwork2015generalization,dwork2015preserving}, Adaptive Data Analysis has seen significant developments, in line with the developments of sophisticated tools for Differential Privacy \cite{dwork2006calibrating}.
The majority of work studying adaptive data analysis focuses on the problem of answering statistical queries.
\cite{bassily2016algorithmic} first showed near-optimal $\tilde{O}(\sqrt{k})$ dependence on the number of adaptively chosen queries.
Subsequently, a sequence of papers has simplified their analysis and techniques and improved their bounds \cite{feldman2018calibrating,jung2019new,dagan2022bounded}.
Recently, \cite{blanc2025subsampling} showed that a much simpler subsampling mechanism actually suffices for optimal adaptive data analysis---despite the fact that it does not suffice for DP.
Indeed, the analysis of \cite{blanc2025subsampling} does not go through the standard ``transfer theorem'' that DP implies adaptive generalization, but rather analyzes the effects of subsampling directly.

\section{Preliminaries}
\label{sec:prelims}

We give formal definitions of the statistical query model of learning \cite{kearns1998efficient} that we adopt, as well as proof systems developed in the context of verifying machine learning \cite{goldwasser2021interactive,mutreja2023pac}.
We start with background on the SQ Model.
Then, we discuss our formalisms for algorithms and proof systems that use SQ oracles.
We defer some technical aspects of our learning model (particularly those related to algorithms' use of randomness) to the relevant technical section (Section~\ref{sec:random-prelims}).
Finally, we review background on adaptive data analysis and the complexity of implementing SQ algorithms from samples \cite{dwork2015generalization}.

\subsection{The Statistical Query Model}

The Statistical Query (SQ) Model \cite{kearns1998efficient} abstracts away the notion of learning from samples to the notion of learning from approximate statistics.
In this model, the learning algorithm is allowed to specify predicates $q:\X \to \{0,1\}$ from a collection\footnote{In our work, we take $\Q= \{0,1\}^\X$ to be the set of all boolean functions, so often drop explicit reference to $\Q$. In this work, we will not consider the computational complexity of evaluating functions $q \in Q$, and instead, measure the complexity of learning in terms of the total number/sequence of queries issued.} $\Q \subseteq \{0,1\}^\X$ and receive---from the \emph{statistical query oracle}---the (rough) expectation of the queried predicate on the distribution of interest $\D$.

\begin{definition}[SQ Oracle, idealized]
    \label{def:oracle}
    A \define{Statistical Query (SQ) Oracle} is a stateful algorithm $\O:\Q \to [0,1]$ that takes input queries $q \in \Q$ and responds with evaluations $\O(q) \in [0,1]$.
    For $\tau > 0$, the oracle $\O$ is \define{\mathdefine{\tau}-accurate over \mathdefine{\D}} if for any finite sequence of $k \in \mathbb{N}$ (adaptively-selected) queries $q_1,q_2,\hdots,q_k$, for all $i \in [k]$
    \begin{gather*}
        \card{\O(q_i) - \E_{X \sim \D}[q_i(X)]} \le \tau.
    \end{gather*}
\end{definition}
Many notable learning algorithms can be described in the SQ Model, including most (Agnostic) PAC Learning algorithms\footnote{Learning Parities is the most notable example of a task that is PAC learnable, but not SQ learnable.} \cite{kearns1998efficient} as well as 
more modern ML algorithms like Gradient Descent.

We remark that, per the definition, the SQ oracle $\O$ with $\tau$-accuracy must always report approximate expectations within $\tau$ of the true expectation.
As such, given access to an SQ oracle, we can hope to design algorithms that always succeed.
In Section~\ref{sec:sq-samples}, we discuss concrete, sample-based implementations of the SQ oracle abstraction, which necessarily introduce a failure probability.

\paragraph{Learning Goal.}
We consider an abstract setting of learning over distributions $\D$ supported on a domain $\X$, such as $\X = \{0,1\}^d$ for some finite dimension $d \in \mathbb{N}$.
Additionally, we consider an abstract collection of hypotheses $\H$.
We define the goal of learning in terms of identifying a hypothesis within some ``good'' set of hypotheses $\Good \subseteq \H$, parameterized by the distribution $\D$.

\begin{definition}[Learning goal, abstract]
    \label{def:learning}
    Fix a class of hypotheses $\H$ and a learning goal $\G\subseteq\H$.
    For a distribution $\D$, we say an algorithm $\A$ \define{learns \mathdefine{\G}} if $\A$ outputs some $g \in \Good$ in the target set of hypotheses.
\end{definition}
For instance, in the context of Agnostic PAC learning, we can take $\H$ to be the concept class $\H \subseteq \{h: \{0,1\}^d \to \{0,1\}\}$, and the good subset $\Good \subseteq \H$ to be all hypotheses achieving classification error $\Pr[g(x) \neq y] \le \Pr[h^*(x) \neq y] + \eps$ competitive with the best $h^* \in \H$ over $\D$.
We adopt this abstract notion of learning, rather than a concrete notion like the PAC framework, in order to emphasize the generality of our approach.
At the extreme, we may take the hypotheses $\H$ to be the collection of sequences of statistical queries and expectations, and the good set $\Good$ to be sequences that arise as a valid execution of a statistical algorithm on $\D$.
(We comment further on our choice of abstract Learning Goal after Definition~\ref{def:protocol} of protocols for delegating SQ learning.)

\paragraph{Statistical Query Algorithms.}
Naturally, we define SQ algorithms as algorithms that may make calls to an SQ oracle.
For SQ algorithms that solve a learning goal, we need to specify the approximation parameter $\tau > 0$ necessary to guarantee correctness.
Formally, we distinguish between deterministic and randomized SQ algorithms.

\begin{definition}[SQ Learning, deterministic]
    Fix a learning goal $\G$ and $\tau > 0$.
    A deterministic algorithm $\A$ \define{\mathdefine{\tau}-SQ learns \mathdefine{\G}} if, for any distribution $\D$ and $\tau$-accurate SQ oracle $\O$ for $\D$, $\A^\O$ outputs a hypothesis $g \in \Good$.
\end{definition}

When $\A$ is deterministic, without loss of generality, we can think of the SQ oracle $\O$ as a pre-specified function where $\O(q)$ is defined for all $q \in \Q$ up front, rather than a stateful algorithm responding sequentially.
Specifically, if the queries that $\A$ issues are a deterministic function of the input to $\A$ and the responses given by $\O$ so far, then a pre-specified function can simulate any stateful oracle, by simply running $\A$ ahead of time.
Note the output of running $\A$ with $\tau$-accurate oracle $\O$ may still be nondeterministic based on the choice of $\O$, even if $\A$ doesn't flip coins.
That said, we insist that given a $\tau$-accurate (idealized) oracle $\O$, the algorithm $\A$ always succeeds.

\paragraph{Randomized SQ Algorithms.}
In our study of SQ proof systems, the distinction between deterministic and randomized algorithms is significant.
In contrast to deterministic algorithms, when $\A$ is randomized, it becomes important to think of the SQ oracle $\O$ as a stateful adversary, who responds with knowledge of prior queries/responses, subject to $\tau$-accuracy.
That is, for a sequence of queries $q_1,q_2,\hdots,q_k$, the response $a_i = \O(q_i)$ may depend on $\langle q_1,a_1,q_2,a_2,\hdots,q_{i-1},a_{i-1},q_i\rangle$.
An SQ algorithm that learns $\G$ must output a good hypothesis with high probability, no matter what decisions the oracle makes in response to the sequence of queries.
\begin{definition}[SQ Learning, randomized]
    Fix a learning goal $\G$ and $\tau>0$, $\gamma>0$. A randomized algorithm $\A$ \define{\mathdefine{(\tau,\gamma)}-SQ learns \mathdefine{\G}} if for any distribution $\D$, and any $\tau$-accurate SQ oracle $\O$ for $\D$, $\A^\O$ outputs a hypothesis $g \in \Good$ with probability at least $1-\gamma$ over the random coins of $\A$.
\end{definition}

With this high-level definition in place, we defer significant details and definitions about randomized SQ algorithms to Section~\ref{sec:random-prelims}, which are essential to understanding our pvCSVs for randomized SQ algorithms and SQ protocols.
In particular, we consider two flavors of SQ oracles for randomized algorithms, ``oblivious'' oracles (as in Definition~\ref{def:oracle}) and ``public-state'' oracles (see Definition~\ref{def:publicstate}).

\paragraph{Complexity Measures for SQ algorithms.}

An important complexity measure of SQ algorithms is the number of statistical queries the algorithm makes $B$, as well as the \emph{adaptivity} to prior query responses $k$.
The adaptivity of an SQ algorithm $\A$ is the number of rounds in which batches of queries are issued to $\O$.
At the extremes of adaptivity, $k = B$ is a fully-adaptive algorithm, where the choice of each query $q_i$ is determined as an arbitrary function of the prior queries and responses $\langle q_1,\O(q_1),q_2,\O(q_2),\hdots,q_{i-1},\O(q_{i-1}) \rangle$; whereas $k = 1$ is a \emph{non-adaptive} SQ algorithm and all of the queries are issued in a single batch.
As we discuss in Section~\ref{sec:sq-samples}, the adaptive query complexity determines how efficiently we can implement the SQ oracle from samples.
\begin{definition}[Query Complexity]
    Fix $\tau > 0$.
An SQ algorithm has \define{\mathdefine{\tau}-query complexity \mathdefine{(k,B)}} if for all distributions $\D$ and $\tau$-accurate SQ oracles $\O$, the execution of $\A^\O$ issues at most $B$ queries over $k$ rounds of adaptivity.
\end{definition}

Additionally, we will informally track the running time of SQ algorithms (eliding details of computation over real-valued responses from the SQ oracle).
\begin{definition}[Running Time]
    Fix $\tau > 0$.
An SQ algorithm has \define{running time \mathdefine{T(\tau)}} if over all $\D$ and all $\tau$-accurate SQ oracles $\O$, the execution of $\A^\O$ runs in at most $T(\tau)$ steps.
\end{definition}

For simplicity's sake, in both query and time complexity, we assume that randomized SQ algoirthms provide a deterministic guarantee on the complexity (that is, a BPP-style guarantee).

\subsection{Proof Systems for Statistical Learning}

We consider proof systems for statistical learning based off of the classic notions from complexity theory and cryptography \cite{gmr,babai1985trading}, and more recently the notion of PAC Verification \cite{goldwasser2021interactive}.
In this setting, a prover $P$ and a verifier $V$ interact in order for the verifier to be convinced that a given hypothesis $g \in \Good$ is good for a learning goal.
We model $P$ and $V$ as randomized algorithms and denote their interaction as $(P,V)$.
We adapt the notions from recent works on PAC Verification to define a model of interactive proof system where the parties have access to SQ oracles as follows.

\begin{definition}[SQ Protocol, adapted from \cite{goldwasser2021interactive,mutreja2023pac}]
\label{def:protocol}
    A \define{Statistical Query protocol} is given by the interaction of two randomized SQ algorithms, the honest prover $P$ and the verifier $V$, denoted $(P,V)$.

    Fix a learning goal $\G$, $\tau_P,\tau_V > 0$, and $\gamma > 0$.
    An SQ protocol $(P,V)$ \define{\mathdefine{(\tau_P,\tau_V,\gamma)}-SQ-verifies \mathdefine{\G}} if for any distribution $\D$ and $\tau_P$-accurate SQ oracle $\O_P$ and $\tau_V$-accurate SQ oracle $\O_V$, the following conditions hold:
    \begin{itemize}
        \item \define{Completeness}:  the honest protocol $(P^{\O_P},V^{\O_V})$ outputs $g \in \Good$ with probability at least  $1-\gamma$ over the random coins of $P$ and $V$.
        \item \define{Soundness}:  for any cheating prover strategy $\tilde{P}$, with probability at least $ 1-\gamma$ over the random coins of $V$, the protocol $(\tilde{P}^{\O_V},V^{\O_V})$ outputs $g \in \Good$ or $V$ rejects.
    \end{itemize}

    For \define{failure probability \mathdefine{\delta}} $> 0$,
    the SQ protocol has \define{sample complexity \mathdefine{(m_P,m_V)}} if the honest prover's oracle $\O_P$ can be implemented from at most $m_P$ samples and the verifier's oracle $\O_V$ can be implemented from at most $m_V$ samples, with probability at least $1-\delta$ over i.i.d.\ samples from $\D$.
\end{definition}

In other words, in an SQ protocol, there exists an honest prover strategy that allows the verifier to accept a good hypothesis $g \in \G(\D)$ with high probability, and conversely, if the verifier accepts a hypothesis $\tilde{g}$, then with high probability,\footnote{For simplicity, we use a single parameter $\gamma$ for the completeness and soundness error.
Of course, it may also be interesting to consider protocols that achieve different completeness and soundness parameters.
In fact, all of our protocols achieve completeness $1$ in the SQ oracle model.} the hypothesis $\tilde{g} \in \G(\D)$, even if it was generated through an interaction with a different prover.

Note that, as is standard, we imagine the cheating prover $\tilde{P}$ is all-powerful and may have arbitrary knowledge of the distribution $\D$.
Additionally, we equip the cheating prover \emph{with the verifier's oracle} $\O_V$; that is, for any query that the verifier asks $q \in \Q$, the precise value $a = \O_V(q)$ is known to both $V$ and $\tilde{P}$.
(That said, the verifier may choose to keep queries private from the prover.)

There are many properties of SQ protocols that may be of interest.
For instance, we consider both private- and public-coin protocols (or, as hinted above, \emph{private-} and \emph{public-query} protocols).
We defer these definitions---particularly those related to our construction of pvCSVs---to Section~\ref{sec:random-prelims}.

Implementing the oracles from samples will necessarily incur some additional failure probability $\delta > 0$, distinguished from soundness error $\gamma$ of the original protocol.
Accounting for the failure probability $\gamma$ of the oracle-based protocol separately from the failure probability $\delta$ that arises from sampling actually leads to improved analysis of the soundness of some of our protocols.
We discuss background on the sample complexity of adaptive SQ algorithms in Section~\ref{sec:sq-samples}.

\paragraph{Comparison to PAC Verification.}
Definition~\ref{def:protocol} is directly inspired by the notion of PAC Verification of statistical algorithms \cite{goldwasser2021interactive,mutreja2023pac}, but departs in a few key ways.

First off, our notion of learning is not tied directly to Agnostic PAC learning (or loss minimization as in \cite{mutreja2023pac}).
Instead, we elect to use our abstract learning goal (Definition~\ref{def:learning}) as the basis for SQ Protocols.
We make this choice because the protocols we design are not actually tied to any properties of loss minimization, but instead run a direct simulation of a given SQ algorithm.
So, provided we start with a learning algorithm $\A$ that achieves its learning goal with good probability, we can turn it into a SQ protocol that also learns $\G$.\footnote{One may even be inclined to define delegation of learning as a distributional simulation of a given algorithm $\A$.  In fact, once formalized, many of our protocols would satisfy such a notion, but we encounter issues with soundness when we want to achieve non-interactive protocols (i.e., pvCSVs) for randomized SQ algorithms in Section~\ref{sec:randomized}.}

Second, we define SQ protocols in the SQ oracle model, rather than in terms of direct samples from the distribution.
In reality, we will be interested in understanding the statistical resources (i.e., samples) necessary to implement the verifier and honest prover strategies.
As in the earlier works on PAC verification, the goal is to design SQ protocols where the verifier's oracle $\O_V$ can be implemented more efficiently than the honest prover's $\O_P$.
For instance, quantitatively, perhaps $P$ requies an oracle with much tighter tolerance than $V$,  $\tau_P < \tau_V$; or as in our work, qualitatively, $P$ may require an adaptive oracle, whereas $V$ uses a non-adaptive oracle.
That said, nothing precludes a hybrid definition, where the verifier only has SQ oracle-access to $\D$, but the prover has more refined access through samples or point evaluations..

\paragraph{Computational Soundness and The Random Oracle Model.}
Definition~\ref{def:protocol} assumes that cheating provers are computationally-unbounded.
We refer to this property as \emph{statistical soundness}.
In contrast, we say that \emph{computational soundness} holds when we prove soundness against probabilistic polynomial-time cheating provers $\tilde{P}$.
When we consider computational soundness, we will still assume that $\tilde{P}$ may hold prior, detailed knowledge of $\D$ (so need not be given an SQ oracle), but does not have the computational resources to break cryptographic primitives.

We will be particularly interested in computational soundness when designing pvCSVs for randomized SQ algorithms.
To do so, we leverage a key cryptographic paradigm:
The Fiat-Shamir Transformation \cite{fiat1986prove}.
The heuristic takes a public-coin interactive protocol and removes interaction by emulating the verifier's random messages with calls to a cryptographic hash function.
Classically, the soundness of Fiat-Shamir is proved in the Random Oracle Model (ROM) \cite{bellare1993random}, which models public access to an (idealized) hash function.
\begin{definition}[Random Oracle Model]
    For output size $m \in \mathbb{N}$, the \define{Random Oracle Model} augments the base computational model by assuming that all parties have access to a public, uniformly-random function $f:\{0,1\}^* \to \{0,1\}^m$ that can be evaluated at unit cost.
\end{definition}
Note that when we consider protocols that operate in the ROM, the probability of violating computational soundness is over the verifier's random coins as well as the draw of the random oracle.

While, in full generality, soundness does not transfer from the ROM to realizable models of computation \cite{barak2001go,goldwasser2003security}, proofs of soundness in the ROM are generally considered as strong evidence of security and sufficient for many practical applications.
Our pvCSVs for randomized SQ algorithms will be proved computationally-sound in the ROM.

\subsection{Sample-Based Implementation of SQ Oracles}
\label{sec:sq-samples}

To run SQ learning algorithms/protocols, we need a concrete implementation of the SQ oracle abstraction based on samples drawn from $\D$.
The number of samples required to guarantee statistical validity of the oracle's responses depends on features of the SQ algorithm $\A$.

In particular, the sample complexity depends on the number of queries issued $B$, the desired accuracy $\tau$, as well as the \emph{adaptivity} $k$.
The adaptivity of an SQ algorithm $\A$ is the number of rounds in which batches of queries are issued to $\O$.
At the extremes of adaptivity, $k = B$ is a completely adaptive algorithm, where the choice of each query $q_i$ is determined as an arbitrary function of the prior queries and responses $\langle q_1,\O(q_1),q_2,\O(q_2),\hdots,q_{i-1},\O(q_{i-1}) \rangle$; whereas $k = 1$ is a \emph{non-adaptive} SQ algorithm and all of the queries are issued in a single batch.
The following definition captures these dependencies.

\begin{definition}[SQ Oracle implementation]
    An algorithm $\O$ is a \define{\mathdefine{(k,B,\tau,\delta)}-implementation of an SQ oracle for a distribution \mathdefine{\D}}, if $\O$ can support any sequence of queries $q_1,q_2,\hdots,q_B$ satisfying the following properties:
    \begin{itemize}
        \item \define{Rounds of adaptivity \mathdefine{k}}:  the sequence of queries is issued in at most $k$ batches; queries selected in the $i$th batch may depend arbitrarily on the queries and responses of the first $i-1$ batches
        \item \define{Query budget \mathdefine{B}}:  the total number of queries in the sequence is upper bounded by $B$
        \item \define{Accuracy \mathdefine{\tau}}:  for every query $q$ in the sequence, the response $\O(q)$ is $\tau$-accurate on $\D$.
        \item \define{Failure probability \mathdefine{\delta}}: for all such query sequences of complexity $(k,B)$, with probability at least $1-\delta$, all queries are answered $\tau$-accurately.
    \end{itemize}
    The \define{sample complexity} of an SQ oracle implementation is the number $m \in \mathbb{N}$ of i.i.d.\ samples from $\D$ required to guarantee $(k,B,\tau,\delta)$-implementation.
\end{definition}
We define the notion of an SQ Oracle implementation in terms of the rounds of adaptivity and the total number of queries, because certain mechanisms are able to exploit limited adaptivity to achieve better sample complexity \cite{hardt2010multiplicative,dwork2015preserving}.
Note that, in order to guarantee $\tau$-accurate queries, we can assume that our implementation returns values using $O(\log(1/\tau))$ bits of precision.

Crucially for our work, the sample complexity required to answer a non-adaptive batch of queries is an exponential improvement over the complexity required to answer an adaptively-selected sequence of queries.
For a non-adaptive batch of queries, the empirical SQ oracle that takes $m$ samples $x_1,\hdots,x_m \sim \D$ and reports $\O(q) = \frac{1}{m} \sum_{i=1}^m q(x_i)$ achieves logarithmic dependence on the number of queries.
\begin{proposition}[Non-adaptive Sample Complexity]
    \label{prop:non-adaptive}
    For any $\tau,\delta>0$, the empirical oracle is a $(1,B,\tau, \delta)$-implementation of an SQ oracle for $\D$ with sample complexity $$m \le O\left(\frac{\log(B/\delta)}{\tau^2}\right).$$
\end{proposition}
Proposition~\ref{prop:non-adaptive} follows by a standard application of Hoeffding's inequality.
The sample complexity for answering an adaptively-selected sequence of statistical queries, however, requires much more sophisticated analysis.
The complexity was only established in the past decade, after connecting the problem of adaptive data analysis with the technique of Differential Privacy \cite{dwork2006calibrating}.

In the case of fully-adaptive algorithms, where $k = B$ statistical queries are issued adaptively, the (roughly) optimal sample complexity scales with $\sqrt{k}$, rather than $\log(k)$.
Initially, the upper bounds followed by leveraging the stability properties of differentially-private query release \cite{bassily2016algorithmic,dagan2022bounded}, but the most recent result of \cite{blanc2025subsampling} leverages a subsampling technique without DP.\footnote{Note that the bounds achieved by \cite{dagan2022bounded} and \cite{blanc2025subsampling} are technically incomparable, but differ only in poly-logarithmic factors in $k$ and $\tau$.
Either mechanism could be used to implement our honest prover's SQ oracle.}

\begin{proposition}[Theorem 3 of \cite{blanc2025subsampling}, Adaptive Sample Complexity Upper Bound]
    \label{prop:full-adaptive-upper}
    For any $\tau,\delta>0$, there exists a mechanism that is a $(k,k,\tau, \delta)$-implementation of an SQ oracle for $\D$ with sample complexity
    \begin{gather*}
        m\leq O\left(\frac{\sqrt{k\cdot \log(k/\delta)\cdot \log(1/\delta)}}{\tau^2}\right)
    \end{gather*}
\end{proposition}
The lower bound holds based on the construction of (interactive) fingerprinting codes \cite{hardt2014preventing,steinke2015interactive}, and holds unconditionally in large-domain settings and assuming the existence of one-way functions in all settings.
\begin{proposition}
[Theorems 1 \& 2 of \cite{steinke2015interactive}, Adaptive Sample Complexity Lower Bound]
    \label{prop:full-adaptive-lower}
    For any $\tau\in(0,0.49)$, any mechanism which is a $(k,k,\tau, 1/2)$-implementation of an SQ oracle for $\D$ has sample complexity \[m \geq \Omega(\sqrt{k})\] if either of the following assumptions hold: one-way functions exist, and the mechanism is computationally bounded; or the space of samples $\X$ is sufficiently large: $|\X|\geq 2^{O(k)}$.
\end{proposition}
In other words, there is a \emph{provable} exponential gap in the sample complexity required to answer $k$ queries adaptively versus non-adaptively.

\section{Publicly-Verifiable Certificates of Statistical Validity}
\label{sec:pvcsv}
\newcommand{\Validate}{\mathsf{Validate}}

In this section, we introduce the main conceptual contribution of this work---\emph{publicly-verifiable Certificates of Statistical Validity}---which allows a prover to publish a digest from the execution of a learning algorithm that \emph{any} verifier can subsequently test for statistical validity.
Importantly, the verifier of a publicly-verifiable Certificate of Statistical Validity (pvCSV) need not hold the same distribution as the prover.
Instead, the verifier can test for statistical validity with respect to their own distribution $\D_V$, through an SQ oracle or samples.

To formalize pvCSVs, we must first introduce the notion of soudness that makes proofs ``publicly-verifiable.''
We define \emph{universal soundness} to ensure soundness holds no matter what distribution the verifier holds.

\begin{definition}[Universal Soundness for SQ protocols]
    Fix a learning goal $\G$, $\tau_V > 0$, and $\gamma>0$.
    An SQ protocol $(P,V)$ has \define{\mathdefine{\gamma}-universal soundness} if for any verifier distribution $\D_V$ and $\tau_V$-accurate SQ oracle $\O_V$ for $\D_V$, and for any prover $\tilde{P}$, with probability at least $1-\gamma$, $(\tilde{P}^{\O_V},V^{\O_V})$ outputs $g \in \GoodV$ or $V$ rejects.
\end{definition}

It is easy to see that universal soundness of SQ Protocols is actually just a restatement of standard soundness; syntactically, our definition simply renames the distribution $\D$ to be in terms of $\D_V$.
But importantly, the restatement of the property allows us to reason rigorously about proof systems when the prover and verifier hold different distributions.
Under universal soundness, when the verifier has access to $\D_V$, running the protocol results in $V$ rejecting or $V$ accepting a hypothesis that is good for the verifier's distribution, regardless of the prover's distribution $\D_P$.

In a bit more detail,
we can imagine an honest prover with oracle access to $\O_P$ for a distribution $\D_P$, while a verifier has access to an oracle $O_V$ for a completely different distribution $\D_V$.
In effect, we can view the honest prover $P^{\O_P}$, as a possible cheating prover for any downstream verifier $V^{\O_V}$.
Still, if the verifier accepts the proof of some non-$\bot$ hypothesis, then universal soundness guarantees that $h \in \GoodV$.
The appeal of universal soundness shows itself when we consider non-interactive proofs, where a prover wants to publish a fixed certificate, which any verifier (holding any $\D_V$) can verify for themselves.
In this context, we can define pvCSVs as a non-interactive SQ protocol satisfying universal soundness.

\begin{definition}[pvCSV]
    \label{def:pvcsv}
    Fix a learning goal $\G$, $\tau_P,\tau_V > 0$, and $\gamma>0$. 
    A \define{publicly-verifiable Certificate of Statistical Validity (pvCSV) scheme} is a non-interactive SQ protocol, where a prover $P$ with 
$\tau_P$-accurate SQ oracle for distribution $\D_P$ publishes a hypothesis $h \in \H$ with certificate $\pi$.
    The pvCSV \define{\mathdefine{(\tau_P,\tau_V,\gamma)}-certifies \mathdefine{\G}} if any
verifier $V$ with
    $\tau_V$-accurate SQ oracle $\O_V$ for
distribution $\D_V$ can read $(h,\pi)$ and accept or reject with the following guarantees:
\begin{itemize}
        \item \define{Completeness}:  if $\D_P = \D_V$, the honest prover, with $\tau_P$-accurate SQ oracle for $\D_P$, can generate a hypothesis-certificate pair $(h,\pi) \gets P^{\O_P}$ such that $V^{\O_V}$ accepts and $h \in \GoodV$ with probability at least $1-\gamma$, over the random coins of $P$ and $V$.
        \item \define{Universal Soundness}:  for any verifier distribution $\D_V$, for any prover strategy $(h,\pi) \gets \tilde{P}^{\O_V}$, $V^{\O_V}$ rejects or $h\in\GoodV$ with probability at least $1-\gamma$ over the random coins of $V$.
    \end{itemize}
\end{definition}

The construction of a pvCSV allows a prover with access to a distribution $\D_P$ to execute an algorithm $\A$ that learns $\G$ once and publish the results $(h,\pi)$.
Then, any party who is interested in running $\A$ on their own distribution $\D_V$ can, instead, run the verification algorithm $V$ on $(h,\pi)$.
If $(h,\pi)$ is accepted, then by universal soundness, $h \in \GoodV$ is guaranteed to be valid for the verifier's distribution.

\paragraph{Remarks.}
A few remarks about the definition of pvCSVs are in order.
\begin{itemize}
    \item \emph{Distributional Access:}~~We describe our pvCSVs as SQ protocols, then give sample-based implementations.
    That said, pvCSVs could equally encompass non-interactive protocols for verification of learning, with more general distributional access, provided we define a more general notion of universal soundness.
    As in \cite{goldwasser2021interactive}, the type of statistical access available may serve as a qualitative difference in the complexity of the prover and verifier.
\item \emph{Complexity Measures:}~~Quantitatively, there are many measures of complexity one could track for pvCSVs.
    We will be most interested in the sample complexity (using i.i.d.\ random draws) of the honest prover and verifier from $\D_P$ and $\D_V$, respectively.
    In general, we may also track other natural complexity measures, such as the time complexity or size of the certificate.
    \item \emph{Computational Soundness:}~~Definition~\ref{def:pvcsv} considers universal soundness with respect to computationally-unbounded provers.
    We also consider computationally-sound pvCSVs, where soundness only holds with respect to PPT algorithms.
    Concretely, we will prove soundness in the Random Oracle Model, where the failure probability is taken over the random draw of random oracle (in addition to the verifier's coins).
\item \emph{On Robust Completeness?}~~
    Note that in our notion of completeness, we assume that the honest prover and verifier have access to exactly the same distribution $\D_P = \D_V$.
    Given the motivating perspective on pvCSVs, we might hope for a weaker notion of distributional closeness required for completeness.
    Such a guarantee may be formally defined in future work.
    We remark that---paired with universal soundness---we still get some notion of robustness.
    If the verifier accepts $(h,\pi)$ then the distributions are indistinguishable, according to the queries asked by (a specific execution of) $\A$, regardless of how close they are in any standard notion of statistical/computational indistinguishability.
\end{itemize}

\subsection{A pvCSV for Every Deterministic SQ Algorithm}

\label{sec:deterministic}

Our first construction establishes that every deterministic SQ algorithm has a statistically-sound pvCSV where the verifier answers $k$ queries non-adaptively and has sample complexity $O(\log(k))$, whereas the best known, efficient prover must answer $k$ possibly adaptive queries with sample complexity scaling with $\tilde{O}(\sqrt{k})$.
We present the pvCSV scheme as a non-interactive SQ protocol, Protocol~\ref{prot:det}.
To generate the certificate, the honest prover first simulates the execution of $\A$ using the prover's SQ oracle $\O_P$, returning a hypothesis $h$.
Along the way, the prover writes down the list of queries $\vec q$ and answers $\O_P(\vec q)$ provided by the oracle as the certificate.
Then, the verifier also simulates the execution of $\A$, but uses the digest of query-answer pairs $(\vec q, \O_P(\vec q))$ provided by the certificate, \emph{rather than making any adaptive statistical queries}.
Finally, the verifier runs a \emph{batch} statistical validation of the answers provided in the certificate using its oracle $\O_V$ for $\D_V$.
The key insight is that the verifier can non-adaptively check every (adaptively-chosen) query produced during the execution of $\A$.

\begin{algorithm}
\renewcommand{\algorithmcfname}{Protocol}\caption{pvCSV for Determinstic SQ Algorithm}
\label{prot:det}
\SetKwInOut{Input}{Input}
\SetKwInOut{Output}{Output}

\textbf{Setup:}
Let $\G,\tau$, $\gamma$ and $\A$ be as described in Theorem~\ref{thm:deterministic}.\\
Let $\O_P$ be a $\tau/3$-accurate SQ oracle for $\D_P$.\\
Let $\O_V$ be a $\tau/3$-accurate SQ oracle for $\D_V$.

\BlankLine
\BlankLine

\BlankLine
\BlankLine

\emph{Simulation of SQ algorithm $\A$} 

\textbf{Prover $P$:}

Initialize an empty transcript $\pi = \langle \rangle$

Simulate $\A$ by executing $\A^{\O_P}$; specifically:
\begin{compactitem}
    \item whenever $\A$ issues an SQ $q$, $P$ queries $\O_P(q)$ and updates $\pi \gets \langle \pi,q,\O_P(q)\rangle$
\end{compactitem}

\uIf{simulation of $\A$ returns $h$}{
    $P$ sends $(h,\pi)$ to $V$}

\BlankLine
\BlankLine

\emph{Verification of transcript}

\textbf{Verifier $V$:}

Read and typecheck $(h,\pi)$.
Parse $\pi$ into $(\vec q, \vec a)$ as follows:
\begin{compactitem}
    \item $\vec q:=(q_1,q_2,\ldots,q_{B_V})$, the list of SQs issued during simulation of $\A$
    \item $\vec a := (a_1,a_2,\ldots, a_{B_V})$, the answers to $\vec q$, where $a_j$ should equal $\O_P(q_j)$
\end{compactitem}

Check that $\pi$ is consistent with $\A$ outputting $h$; that is,
\begin{compactitem}
    \item Simulate $\A$ using $\vec a$ to answer SQs
    \item Check that $\A$ returns $h \neq \bot$
\end{compactitem}

\uIf{$\pi$ is NOT consistent with $\A$}{
Reject and output $\bot$
}

\BlankLine
\BlankLine

\emph{Statistical Validation of simulated oracle queries}

\uIf{NOT $\mathsf{Validate}(\vec q, \vec a, \tau/3, \tau)$}
{Reject and output $\bot$}
Accept and output $h$
\end{algorithm}

\begin{algorithm}
\label{alg:validate}
\caption{$\mathsf{Validate}(\vec q, \vec a, \tau_P, \tau)$~~ i.e., Non-adaptive Statistical Validation}
\label{alg:sq-check}
\SetKwInOut{Input}{Input}
\SetKwInOut{Setup}{Setup}

\Input{Batch of statistical queries $\vec q$;\\ List of candidate answers to queries $\vec a$;\\
Candidate accuracy $\tau_P > 0$ and final accuraccy $\tau > 0$}
\Setup{Non-adaptive $\tau_V$-accurate SQ oracle $\O_V$ s.t.\ $\tau_P + 2\tau_V \le \tau$}

Issue batch of queries $\vec q$ to $\O_V$; let $\hat{a}_j:=\O_V(q_j)$ for all $1\leq j\leq B$

\For{$j = 1,\hdots,B$}{
    \uIf{$|\hat{a}_j-a_j| > \tau_V + \tau_P$}
    {\Return False}
}
\Return True
\end{algorithm}

We prove that Protocol~\ref{prot:det} is a pvCSV where the verifier's sample complexity scales ``non-adaptively'' even when an adaptive algorithm is necessary for learning.
We start with the following lemma.

\begin{lemma}
    \label{lem:deterministic_sq}    
    Fix a learning goal $\G$, $\tau>0$. Suppose $\A$ is a deterministic algorithm that $\tau$-SQ learns $\G$. There is an SQ protocol $(P,V)$, described in Protocol~\ref{prot:det}, with the following properties:
    \begin{enumerate}[(a)]
        \item $(P,V)$ is a non-interactive SQ protocol that $(\tau/3,\tau/3,0)$-SQ verifiers $\G$
        \item $V$ uses a non-adaptive SQ oracle, regardless of the SQ adaptivity of $\A$. That is, suppose $\A$ has $\tau$-query complexity $(k,B)$; then,
        \begin{itemize}
            \item $P$ has $\tau/3$-query complexity $(1,B)$;
            \item $V$ has $\tau/3$-query complexity $(k,B)$
        \end{itemize}
    \item $V$ and $P$ run in linear time in the original algorithm. That is, suppose $\A$ runs in time $T_\A(\tau)$; then both $P$ and $V$ run in time $O(T_\A(\tau))$
    \end{enumerate}
\end{lemma}

\begin{proof}
    First, we analyze the prover and verifier's query complexities.
    Observe that $P$ runs a direct simulation of the queries issued by $\A$, supported by $\O_P$.
    Since $\O_P$ is $\tau/3$-accurate, then $P$ has $\tau/3$-query complexity $(k,B)$.
    In contrast, $V$ calls $\O_V$ once on all $B$ queries through the $\Validate$ subroutine. Since $\O_V$ is $\tau/3$-accurate, $V$ has $\tau/3$-query complexity $(1,B)$. So part (b) holds.

    Next, we analyze run time. 
    Observe that both $P$ and $V$ simulate an execution of $\A$. While the oracles $\O_P$ and $\O_V$ are both $\tau/3$-accurate, the simulations of $\A$ can assume $\tau$-accuracy, so the simulations take time $O(T_\A(\tau))$. $V$ also runs a statistical valdiation that checks the statistical closeness of the $B$ queries, each of which takes constant time to evaluate in the query model, which is also linear in the run time of $\A$. Thus, part (c) holds.

    To complete the lemma, we show that $(P,V)$ is a non-interactive SQ protocol that $(\tau/3,\tau/3,0)$-SQ verifies $\G$.\footnote{The failure probability of the pvCSV will arise when we implement the protocol's oracles using samples.} Observe by construction $(P,V)$ is non-interactive: $P$ sends one message $(h,\pi)$ to $V$. Next, we show $(P,V)$ satisfies perfect completenees and universal soundness for $\G$
    
    \emph{Completeness:}~~ We consider the honest protocol $(P,V)$ where $\D_P=\D_V$. First, we show that $V$ always accepts the honestly generated certificate $(h,\pi)$. Since $(h,\pi)$ is honestly generated, $V$ can parse $\pi$ into queries $\vec q$ and answers $\vec a$. Then it is sufficient to show that $\pi$ is consistent with an execution of $\A$ that results in the output $h$.
    Consider the sequence of queries $\vec q = (q_1,q_2,\hdots,q_B)$ listed in the certificate.
    We claim this is exactly the same sequence of queries that $V$ will produce when simulating $\A$.
    This claim follows (formally by induction on the sequence of queries) because $\A$ is deterministic and the verifier's simulation of $\A$ uses the same query responses $\vec a = (\O_P(q_1),\O_P(q_2),\hdots,\O_P(q_B))$ as the prover's execution.
    By the same argument, the output $h$ of the prover and verifier's execution of $\A$ will be the same (since $\A$ is deterministic and operates on the same input sequence).

    Second, we show that $\pi$ passes the verifier's statistical validation $\Validate(\vec q, \vec a, \tau/3,\tau)$, according to $\O_V$.
    Recall that $\O_P$ and $\O_V$ are both $\frac{\tau}{3}$-accurate oracles for $\D_P = \D_V$, by assumption.
    Since the verifier issues the same queries $q_1,q_2,\hdots,q_B$ to $\O_V$, we see that, for each $1 \le j \le B$, the comparison looks at the difference between $\O_P(q_j)$ and $\O_V(q_j)$.
    Both queries are within ${\tau}/{3}$ of the true expectation $\E_{x \sim \D}[q_j(x)]$, so by the triangle inequality $|\hat{a}_j - a_j| \le {2\tau}/{3}$. Therefore, $\Validate$ always returns true, and $V$ always accepts the honestly generated $(h,\pi)$.

    So it is sufficient to show that $h\in\GoodV$.
    Recall that $h$ is the output of running $\A$ using a ${\tau}/{3}$-accurate oracle for $\D_P$.
    Because $\D_P = \D_V$ (and ${\tau}/{3} \le \tau$), this execution is also the result of running $\A$ using \emph{some} $\tau$-accurate oracle for $\D_V$.
Thus, by the assumption that $\A$ $\tau$-SQ learns $\G$, then $h \in \GoodV$. So, $(P,V)$ achieves completeness $1$.

    \emph{Universal Soundness:}~~
Let $\D_V$ be any distribution, and $\O_V$ be any $\tau/3$-accurate SQ oracle for $\D_V$.
    Take any cheating prover strategy $\tilde{P}$, and consider the interaction $(\tilde{P}^{\O_V},V^{\O_V})$.
    It suffices to argue, as in the above argument for completeness, that if $V$ accepts, then $h \in \GoodV$.
    If $V$ accepts, then it must be the case that $\pi$ is consistent with some execution of $\A$ that results in outputting $h$.
    Further, $\Validate(\vec q, \vec a, \tau/3, \tau)$ passes, which implies the answers $\vec a$ are within ${2\tau}/{3}$ of the responses from the verifier's non-adaptive oracle query $\hat{\vec a}=\O_V(\vec q)$.
Since $\O_V$ is $\tau/3$-accurate with respect to $\D_V$, by the triangle inequality, $|a_j-\E_{x\sim \D_V}[q_j(x)]|\leq \tau$ for all $1 \le j \le B$.

    In combination, $\pi$ gives the verifier a way to simulate $\A$ using $\tau$-accurate answers its queries over $\D_V$.
The learning guarantee of $\A$ implies that the returned hypothesis $h\in\GoodV$ must be valid.
Thus under interaction with any prover cheater, the verifier $V$ either rejects or accepts and outputs $h\in\GoodV$, and $(P,V)$ satisfies universal soundness $0$.
\end{proof}

With this lemma in place, establishing our main theorem of this section is straightforward.
To achieve an improved bound on the size of the certificate, we make the following optimization.
Notice that our proof of correctness does not actually use the fact that the queries $\vec q$ used by the prover are the same as the queries derived by the verifier's simulation of $\A$; instead, we only need that the answers in the certificate, $\vec a$ are valid for the verifier's queries.
Thus, we do not need to include queries in the certificate at all: we need only the certificate be the list of answers, that is, $\pi:=\vec a$.
While including the names of the queries improves the clarity of the proof system,
we derive an equally-valid SQ protocol where $\pi:=\vec a$ in the following theorem.

\begin{theorem}[Formal Statement of Theorem~\ref{result:deterministic}]
    \label{thm:deterministic}
Fix a learning goal $\G$ and $\tau > 0$. Suppose $\A$ is a deterministic SQ algorithm that $\tau$-SQ learns $\G$, which makes $k$ adaptive queries to a $\tau$-accurate oracle $\O$.
    For any $\delta > 0$,
    there exists a pvCSV scheme for $\G$ with failure probability $\delta$ achievable in the following complexities.
    \begin{itemize}
\item Honest prover sample complexity:  $m_P \le O\left({\sqrt{k\cdot \log(k/\delta)\cdot \log(1/\delta)}}/{\tau^2}\right)$
        \item Verifier sample complexity:  $m_V \le O(\log(k/\delta)/\tau^2)$
        \item Certificate size:  $\card{\pi} \le O(k \cdot \log(1/\tau))$
    \end{itemize}
\end{theorem}

\begin{proof}
Let $(P,V)$ be  Protocol~\ref{prot:det}.
By Lemma~\ref{lem:deterministic_sq}, to run a fully-adaptive SQ algorithm using $k$ queries, it suffices to have a $(k,k,{\tau}/{3})$-implementation of $\O_P$ and a $(1,k,{\tau}/{3})$-implementation of $\O_V$.
The honest prover and verifier sample complexities follow from Proposition~\ref{prop:full-adaptive-upper} and Proposition~\ref{prop:non-adaptive}, respectively.
The original SQ verification protocol contributes $\gamma = 0$ failure probability, so by union bounding over the failure probability of implementing $\O_P$ and $\O_V$, completeness and universal soundness hold with probability at least  $1-\delta$.

Finally, we analyze the size of the certificate $\pi$.
Without loss of generality, we modify the certificate so that $\pi:=\vec a$, and the verifier so that it doesn't check for $\tilde{q}_i=q_i$. As mentioned above, this does not change the validity of Lemma~\ref{lem:deterministic_sq}.
Every answer to a statistical query requires $O(\log(1/\tau))$ bits of space. Since $\pi=\vec a=(a_1,a_2,\ldots,a_B)$, then $|\pi|=O(B\log(1/\tau))$.    
\end{proof}

Finally, we remark informally on the time complexities of the honest prover and verifier.
Both parties run a direct simulation of the original SQ algorithm $\A$.
Additionally, the honest prover must implement the mechanism of \cite{blanc2025subsampling} for adaptive data analysis to answer the queries to $\O_P$, whereas $\O_V$ may respond with the empirical statistics from its sample.
In most settings, the overall running time for both will be dominated by the simulation of $\A$.

\section{Delegation of Randomized SQ Learning}
\label{sec:randomized}

In this section, we address delegating SQ learning algorithms that use randomness.
Our investigation leads us to develop a number of new definitions and tools for understanding SQ protocols.
With these tools in place, we obtain (computationally-sound) pvCSVs, not only for randomized SQ algorithms, but for a large class of SQ protocols.

We start by giving essential definitions, highlighting important properties of randomized SQ algorithms and protocols.
Using these definitions, we then show that for a large class of SQ protocols, there exists a reformulation into a Canonical SQ Protocol.
This canonical protocol allows for non-adaptive verification of the statistical queries, leading to exponential savings in sample complexity for the verifier, for any original SQ protocol (including randomized algorithms that make adaptive statistical queries).
We conclude the section with a construction of pvCSVs for this class, by showing how to apply a Fiat-Shamir transformation to the canonical protocol.

\subsection{Definitions for Randomized SQ Learning and SQ Protocols}
\label{sec:random-prelims}

To begin, we fix a standard, mechanical format for describing randomized SQ algorithms that proceed in \emph{epochs}.
In the $t$-th epoch,
\begin{itemize}
    \item $\A$ samples randomness $r_t \in \{0,1\}^{d_t}$, for some (efficiently-bounded) $d_t \in \mathbb{N}$
    \item $\A$ then executes a deterministic SQ algorithm $\A_t$ as a subroutine, with $r_t$ as fixed input, which may issue adaptive statistical queries.
\end{itemize}
The choice of the $t$-th subroutine $\A_t$ may be chosen adaptively as a function of prior randomness, as well as any statistical queries issued.
In terms of expressivity, formatting randomized algorithms into epochs is not a restriction, as we can simulate any algorithm in this format.
But the number of epochs turns out to be a key quantity that affects the complexity of our proof systems.
\begin{definition}[Epoch Complexity]
Fix $\tau > 0$ and $\ell\in\Nbb$.
A randomized SQ algorithm $\A$ has \define{epoch complexity \mathdefine{\ell}} if for any $\tau$-accuate SQ oracle $\O$, the execution of $\A^\O$ can be broken into $\ell$ epochs.
\end{definition}
Extending the notion of epoch complexity to SQ protocols, we standardize the verifier's execution similarly.
In the $t$-th epoch,
\begin{itemize}
    \item $V$ samples randomness $r_t\in\{0,1\}^{d_t}$ for some (efficiently-bounded) $d_t\in\N$
    \item $V$ executes a deterministic, interactive SQ algorithm $V_t$ as a subroutine, with $r_t$ as a fixed input, which may issue adaptive statistical queries and communicate with $P$.
\end{itemize}
Again, the choice of the $t$-th subroutine $V_t$ can be chosen adaptively based on prior randomness, $V$'s internal state, and messages from the prover.

\begin{definition}[Epoch Complexity of SQ Protocols]
Fix $\tau_V > 0$ and $\ell\in\Nbb$.
An SQ verifier $V$ has \define{epoch complexity \mathdefine{\ell}} if for any prover $\tilde{P}$ and any $\tau_V$-accurate SQ oracle $\O_V$, the execution of $V$ within the interaction $(\tilde{P}^{\O_V},V^{\O_V})$ can be broken into $\ell$ epochs.
\end{definition}

Intuitively, formatting randomized algorithms (and SQ verifiers) into epochs allows us to separate ``adaptivity'' of $\A$ based on fresh randomness and adaptivity based on the results of statistical queries (and prover messages).
In the extreme, each epoch could issue a single statistical query, so adaptivity to the randomness and queries are essentially the same.
At the other extreme, we might consider algorithms that start with a random initialization, but then proceed deterministically.
Such algorithms may still make adaptive SQ queries, but could be implemented in a single epoch.
The number of epochs serves as a natural complexity measure of the algorithm's use of its randomness.

\paragraph{Oblivious vs.\ Public-State SQ Oracles.}
Recall that we model an SQ oracle as a stateful algorithm, who may respond adversarially (within $\tau$-accuracy) based on the state of the algorithm $\A$.
When our algorithms leverage randomness, we need to reason about what knowledge the oracle $\O$ has about the algorithm's internal state---namely, the randomness sampled at the current epoch.

One standard notion of SQ oracle allows the algorithm to maintain private state from the oracle.
In this case, the algorithm need not reveal any of its random coins to the oracle.
We refer to such oracles as \emph{oblivious} SQ oracles, which follow the earlier notion given in Definition~\ref{def:oracle}.

For randomized SQ algorithms, we mainly focus on oracles, who may respond adversarially with full knowledge of the algorithm's randomness.
This notion is captured by a \emph{public-state} SQ oracle.
\begin{definition}[Public-State SQ Oracle]
\label{def:publicstate}
    A \define{public-state SQ oracle} is a stateful algorithm $\O:\Q\times \{0,1\}^* \to [0,1]$ that takes as input a query $q\in \Q$ and a string $r\in\{0,1\}^*$ representing random coins, and responds with an evaluation $\O(q,r)\in [0,1]$.

    For $\tau > 0$, the oracle $\O$ is \define{\mathdefine{\tau}-accurate over \mathdefine{\D}} if for any finite, adaptively chosen sequence of $k \in \mathbb{N}$ queries and randomness $(q_1,r_1), (q_2,r_2),\hdots, (q_k,r_k)$, for all $i \in [k]$,
    \begin{gather*}
        \card{\O(q_i,r_i) - \E_{X \sim \D}[q_i(X)]} \le \tau.
    \end{gather*}
    Concretely, consider a sequence issued by a randomized SQ algorithm $\A$.
    In the $t$-th epoch, whenever $\A^t$ issues a query $q \in \Q$, the randomness $r_1 \hdots r_t$ is sent to the public-state oracle $\O(q,r_1 \hdots r_t)$.

    Further, a randomized SQ algorithm \define{\mathdefine{\tau}-SQ learns \mathdefine{\G} with public state}, if correctness holds for any $\tau$-accurate public-state SQ oracle $\O$; an SQ protocol $(P,V)$ \define{\mathdefine{(\tau_P,\tau_V,\gamma)}-SQ verifies \mathdefine{\G} with public verifier state}, if completeness and soundness hold for any $\tau_V$-accurate public-state SQ oracle $\O_V$.
\end{definition}
In other words, a randomized algorithm equipped with a public-state SQ oracle may, as usual, specify the query $q \in \Q$ of interest, but necessarily reveals the randomness sampled so far.
By the standardization into epochs, the randomness captures all of the state of $\A$.

Note that any algorithm/protocol that succeeds with access to a public-state oracle also works with an oblivious oracle, but the converse is not true.
Correctness with respect to a public-state oracle is a stronger guarantee: no matter what information is leaked to the oracle through interaction, the algorithm will work.
That said, proving correctness may be more challenging, so designing correct algorithms with a public-state oracle may be harder than with an oblivious oracle.

\paragraph{Public-Coin SQ Protocols.}
As with traditional interactive protocols, we can distinguish SQ protocols based on the verifier's communication of its randomness.
Key to our study of non-interactive proof systems for randomized SQ algorithms, we first consider \emph{public-coin} interactive SQ protocols.
Defining public-coin protocols when the verifier is equipped with an SQ oracle (which may introduce non-determinism) is nuanced.
We consider two alternative definitions.

\begin{definition}[Public-Coin SQ Protocols]\label{def:public-coin}
    An SQ protocol $(P,V)$ is a (standard) \define{public-coin} SQ protocol if in every round of communication, $V$ samples a uniformly random string $r$, independent of all prior randomness and messages, and sends $r$ to $P$.
    Concretely, for $V$ of epoch complexity $\ell$, the protocol consists of $\ell$ rounds, where in the $t$-th round, $V_t$ sends its randomness $r_t$ to $P$.
\end{definition}

\begin{definition}[Mixed-Message SQ Protocols]
    An SQ protocol $(P,V)$, where $V$ has epoch complexity $\ell$.
    $(P,V)$ is a \define{mixed-message} (public-coin) SQ protocol if for all $t = 1,\hdots,\ell$, the first message of the $t$-th epoch is from $V_t$ to $P$ and includes the epoch's randomness $r_t$.
\end{definition}

In other words, a standard public-coin SQ protocol adopts the formalism that the only messages the verifier sends to the prover are its randomness, whereas a mixed-message SQ protocol must reveal its randomness to the prover, but may also send non-random challenges.
Without an SQ oracle, the distinction is moot: in the $t$-th epoch, $V_t$ is a deterministic algorithm, so the prover can simulate any challenges it would receive and respond accordingly.
But with an SQ oracle---which may respond adversarially within its tolerance---$V_t^{\O_V}$ may have non-deterministic behavior, even conditioned on $r_t$.
Thus, in principle, mixed-message SQ protocols could be more expressive than public-coin SQ protocols.
Despite this distinction, our main result in Section~\ref{sec:canonical} (Protocol~\ref{prot:canon-epoch}) implies that (as in standard interactive proofs) any mixed-message SQ protocol can be compiled into a public-coin SQ protocol with the same completeness and soudness guarantees.

\paragraph{Public-Query SQ Protocols.}
A final consideration in classifying SQ protocols is whether the verifier's queries are kept private or made public to the prover.
A public-query protocol reveals its queries to the prover at every round of communication.
\begin{definition}[Public-Query SQ Protocols]
    An SQ protocol $(P,V)$ is a \define{public-query} SQ protocol if
    in every round $i$ of communication, the $i$th message $m_i$ from $V$ to $P$ includes every statistical query $q_{i_1},\hdots,q_{i_j}$ issued by $V$ to $\O_V$ since the previous message $m_{i-1}$.
\end{definition}
Note that, per the discussion above, whether the verifier's queries are public or private is orthogonal to whether the verifier's coins are public or private.
That said, Protocol~\ref{prot:canon-epoch} also implies that any public-coin, private-query protocol can be made public-query.
In this sense, we use the term ``public-query'' to refer to private-coin, public-query SQ protocols.
We explore public-query protocols further when we consider the strength and limits of SQ proof systems beyond pvCSVs in Section~\ref{sec:limits}.

\subsection{A Canonical Public-Coin SQ Protocol}
\label{sec:canonical}

With the preliminaries on randomized SQ algorithms and SQ protocols in place, we are ready to describe protocols for delegating randomized SQ algorithms.
We describe, in Protocol~\ref{prot:canon-epoch}, an interactive public-coin SQ protocol for delegating any randomized SQ algorithm that SQ-learns a concept $\G$ with public state.
In fact, our protocol is much more general, and can take any mixed-message, private-query SQ protocol that SQ-verifies $\G$ with public verifier state (which include all randomized SQ algorithms with public state) and compile it into a canonical public-coin SQ protocol, where the verifier issues a single non-adaptive batch of statistical queries.

\newcommand{\can}{\mathsf{can}}
\newcommand{\Pcan}{\mathsf{P}_\can}
\newcommand{\tildePcan}{\tilde{\mathsf{P}}_\can}
\newcommand{\Vcan}{\mathsf{V}_\can}
\newcommand{\OPcan}{\O_{\Pcan}}
\newcommand{\OVcan}{\O_{\Vcan}}

\begin{algorithm}
\renewcommand{\algorithmcfname}{Protocol}\caption{Canonical SQ Protocol}
\label{prot:canon-epoch}
\SetKwInOut{Input}{Input}
\SetKwInOut{Output}{Output}

\textbf{Setup:}
Let $\G,\tau_P,\tau_V,\gamma$ and $(P,V)$ be as described in Lemma~\ref{lem:canonical}.\\
Let $\OPcan$ be a $\min\{\tau_P,\tau_V/3\}$-accurate SQ oracle for $\D_P$.\\
Let $\OVcan$ be a $\tau_V/3$-accurate SQ oracle for $\D_V$.

\BlankLine
\BlankLine

\BlankLine
\BlankLine

\textbf{Phase 1:}  \emph{Interactive simulation of SQ protocol $(P,V)$} 

Initialize an empty transcript $\pi_0 = \langle \rangle$

\For{$i=1,\ldots,\ell$}{
    \textbf{Prover} $\Pcan$ and \textbf{Verifier} $\Vcan$ determine $V_i$  for the $i$-the epoch of $V$, based on $\pi_{i-1}$

    \BlankLine

    \textbf{Verifier} $\Vcan$ samples the $i$-th epoch's randomness $r_i$ and sends to $\Pcan$

    \BlankLine

    \textbf{Prover} $\Pcan$ updates $\pi_{i} \gets \langle \pi_{i-1},r_i \rangle$
    
    $\Pcan$ simulates the $i$-th epoch by executing $(P^{\OPcan},V_i^{\OPcan})(\pi_{i})$; specifically:
    \begin{compactitem}
\item when $V_i$ issues SQ $q$, $\Pcan$ queries $\OPcan(q)$ and updates $\pi_i \gets \langle \pi_i,q,\OPcan(q)\rangle$
        \item when $V_i$ sends message $s$ to $P$, $\Pcan$ simulates prover $P^{\OPcan}(\pi_i)$
        \item when $P$ sends message $m$ to $V$, $\Pcan$ updates $\pi_i \gets \langle \pi_i, s, m \rangle$
\end{compactitem}

    \BlankLine
    
    \uIf{simulation of $V$ returns $h$}{$\Pcan$ updates $\pi \gets \pi_i$ and sends $(h,\pi)$ to $\Vcan$
    
    $\Pcan$ breaks from loop}
    \uElseIf{simulation of $V$ ready for next epoch}{
    $\Pcan$ sends updates to transcript $\pi_i$ to $\Vcan$ for next epoch
    }
}

\BlankLine
\BlankLine

\textbf{Phase 2:}  \emph{Verification of interactive transcript}

\textbf{Verifier $\Vcan$:}

Read and typecheck $(h,\pi)$.
Parse $\pi$ into $(\vec q_V, \vec a_V, \tilde{\vec r}, \vec m, \vec s)$ as follows:
\begin{compactitem}
    \item $\vec q_V:=(q_1,q_2,\ldots,q_{B_V})$, the list of verifier SQs issued during simulation of $(P,V)$
    \item $\vec a_V := (a_1,a_2,\ldots, a_{B_V})$, the answers to $\vec q_V$, where $a_j$ should equal $\O_P(q_j)$
    \item $\tilde{\vec r} := (\tilde{r}_1,\hdots,\tilde{r}_\ell)$, the public coins of $V$
    \item $\vec s := (s_1,\hdots,)$, the non-random messages of $V$ to $P$ 
    \item $\vec m := (m_1,\hdots,)$, the messages of $P$ to $V$
\end{compactitem}

Check that $\pi$ is consistent with $(P,V)$ outputting $h$; that is,
\begin{compactitem}
    \item Check that $\tilde{\vec r} = r_1,\hdots,r_\ell$
    \item Simulate $V$ with randomness $\tilde{\vec r}$; use $\vec a_V$ to answer SQs, and $\vec m$ for messages from $P$
    \item Check that $V$ returns $h \neq \bot$
\end{compactitem}

\uIf{$\pi$ is NOT consistent with $(P,V)$}{
Reject and output $\bot$
}

\BlankLine
\BlankLine

\textbf{Phase 3:}  \emph{Validation of simulated oracle queries}

\textbf{Verifier $\Vcan$:}

\uIf{NOT $\Validate(\vec q_V,\vec a_V,\tau_V/3,\tau_V)$}{
    Reject and output $\bot$
}

Accept and output $h$

\end{algorithm}

\paragraph{Protocol Description.}
Protocol~\ref{prot:canon-epoch} works as follows.
We start with a mixed-message SQ protocol $(P,V)$ where the verifier may make private statistical queries, and each player is arbitrarily-adaptive in their queries.
Note that randomized SQ algorithms are the special case of such protocols, where the algorithm must accept $h$ or output $\bot$ without assistance from any prover.
We want to build a public-coin, public-query SQ protocol $(\Pcan,\Vcan)$ where the verifier issues a single batch of statistical queries.
As in Protocol~\ref{prot:det}, the verifier will delegate its queries to the prover and check them at the end.

Because $(P,V)$ is mixed-message, we know that the verifier $V$ starts each epoch by sending its randomness to the prover.
This random message will be the only message $\Vcan$ sends to $\Pcan$ per epoch; then, $\Pcan$ will be responsible for simulating the remaining interactive execution of that epoch.
$\Pcan$ will use its own SQ oracle $\OPcan$ to make any queries.
This aspect of the protocol---where the prover simulates the queries of the verifier with knowledge of its randomness---is where we need to leverage the assumption that $(P,V)$ learns $\G$ with public verifier state; that is, even if the oracle (or in this case the prover) can respond adversarially based on the state of the verifier, the protocol is still sound.
Once the epoch finishes, $\Pcan$ will send the transcript of the execution to $\Vcan$, who can update state and move to the next epoch.

At the end, the canonical verifier $\Vcan$ checks to ensure that the prover faithfully simulated the execution of $P$ and $V$.
Finally, $\Vcan$ runs a statistical check to make sure all of the query values reported by the prover are actually within the required tolerance of the original protocol.

In all, we obtain the following guarantee on our canonical public-coin SQ protocol.

\begin{lemma}[Formal Statement of Lemma~\ref{result:canonical}]
    \label{lem:canonical}
    Fix a learning goal $\G$, $\tau_P,\tau_V > 0$, and $\gamma > 0$, and let $\tau_P' = \min\{\tau_P,\tau_V/3\}$ and $\tau_V' = \tau_V/3$.
    Suppose $(P,V)$ is a mixed-message, private-query SQ protocol that $(\tau_P,\tau_V,\gamma)$-verifies $\G$ with public verifier state.
    There is an SQ protocol $(\Pcan,\Vcan)$, described in Protocol~\ref{prot:canon-epoch}, with the following properties:
\begin{enumerate}[(a)]
        \item $(\Pcan,\Vcan)$ is a public-coin SQ protocol that $(\tau_P',\tau_V',\gamma)$-verifies $\G$
        \item For $V$ of epoch complexity $\ell$, $(\Pcan,\Vcan)$ has at most $\ell$ rounds of interaction.
        \item $\Vcan$ uses a non-adaptive SQ oracle, regardless of the SQ adaptivity in $(P,V)$.
        That is,\\ suppose $V$ has $\tau_V$-query complexity $(k_V,B_V)$ and $P$ has $\tau_P$-query complexity $(k_P,B_P)$; then,
\begin{itemize}
\item $\Vcan$ has $\tau_V'$-query complexity $(1,B_V)$;
\item $\Pcan$ has $\tau_P'$-query complexity $(k_P+k_V,B_P+B_V)$
        \end{itemize}
        \item $\Vcan$ and $\Pcan$ run in linear time in the original protocol.
        That is,\\ suppose $V$ runs in time $T_V(\tau_V)$ and $P$ runs in time $T_P(\tau_P)$; then
        \begin{itemize}
            \item $\Vcan$ runs in time $O(T_V(\tau_V))$
            \item $\Pcan$ runs in time $O(T_V(\tau_V) + T_P(\tau_P))$
        \end{itemize}
        \item For $P$ of communication complexity $c_P$, $\Pcan$ has communication complexity $O(B_V\log(1/\tau_P')) + c_P$.
    \end{enumerate}
\end{lemma}

\begin{proof}
    Before we prove that the canonical protocol verifies $\G$, we will analyze the other properties first. First, we analyze round complexity. Observe by construction that each round of communication corresponds to one epoch of $V$. Since $V$ has epoch complexity $\ell$, $(\Pcan,\Vcan)$ has round complexity $\ell$. So part (b) holds. 
    
    Next, we analyze statistical complexity. Observe that $\Pcan$ simulates the interaction $(P^{\OPcan},V^{\OPcan})$, so $\OPcan$ must answer at most $B_P$ queries from $P$ and $B_V$ queries from $V$, and compose the rounds of adaptivity for $k_P+k_V$ rounds. Thus $\Pcan$ has $\tau_P'$-query complexity $(k_P+k_V, B_P+B_V)$. Meanwhile, $\Vcan$ only checks the statistical validity of the proposed verifier's queries in a single batch to $\OVcan$, so it has $\tau_V'$-query complexity $(1, B_V)$. Therefore part (c) holds.

    Next, we analyze run time. $\Pcan$ just simulates the entire interaction $(P^{\OPcan},V^{\OPcan})$. While $\OPcan$ operates at $\tau_P'$ accuracy, in order to meet the $(\tau_P,\tau_V,\gamma)$-learning guarantee of the original protocol, $\Pcan$ simulates $P$ running at $\tau_P$ accuracy and $V$ running at $\tau_V$ accuracy. So $\Pcan$ runs in time $O(T_V(\tau_V)+T_P(\tau_P))$. Similarly, $\Vcan$ simulates $V$ but using transcript $\pi$ and running at $\tau_V$ accuracy. So $\Vcan$ runs in time $O(T_V(\tau_V))$, and part (d) holds.

    Now, we analyze the $\Pcan$'s communication complexity. In the description of Protocol~\ref{prot:canon-epoch}, $\Pcan$ implicitly sends the entire transcript $\pi$ of the simulation of $(P,V)$, which comprises of the queries $\vec q_V$ and answers $\vec a_V$ of $V$, public coins $\tilde{\vec r}$ of $V$, non-random messages $\vec s$ from $V$ to $P$, and messages $\vec m$ from $P$ to $V$. However, strictly speaking $\Pcan$ does not need to relay $\vec q_V$, $\tilde{\vec r}$, and $\vec s$. $\Vcan$ already has the true public coins $(r_1,\ldots,r_\ell)$, and can simulate $V$ using that, $\vec a_V$ and $\vec m$ to generate the missing $\vec q_V$ and $\vec s$. Therefore, we can simplify $\Pcan$'s communication complexity to just $\vec a_V$ and $\vec s$. Then note $|\vec a_V| =O(B_V\log(1/\tau_P'))$. And $|\vec s|=c_P$, the communication complexity of $P$. So $\Pcan$'s communication complexity is $O(B_V\log(1/\tau_P')+c_P$, and part (e) holds.

    To complete the lemma, we show that \((\Pcan,\Vcan)\) is a public-coin SQ protocol that $(\tau_P',\tau_V',\gamma)$-SQ verifies $\G$. First, observe that $\Vcan$ sends all random coins generated for the $\ell$ rounds of communication, so it is public-coin. As noted in part (c), $\Pcan$ employs a $\tau_P'$-accurate SQ oracle, and $\Vcan$ employs a $\tau_V'$-accurate SQ oracle. So what remains to be shown is that the canonical transformation preserves $1-\gamma$ completeness and $\gamma$ soundness.

    \emph{Completeness:}~~ Assume $\D_P=\D_V$, and consider the honest protocol $(\Pcan,\Vcan)$ whose interaction generates $(h,\pi)$. First, we note that in the honest protocol, $\pi$ is well-formatted, so the canonical verifier can parse $\pi$ into $(\vec q_V, \vec a_V,\tilde{\vec r},\vec s,\vec m)$ and successfully check that $\pi$ is consistent with an execution of $(P,V)$ that outputs $h$. That is, $V$, using randomness $\tilde{\vec r}$ receiving evaluations $\vec a_V$ and messages $\vec m$, will make queries $\vec q_V$, send messages $\vec s$, and output the hypothesis $h$. This works since $\Pcan$ will properly simulate $(P,V)$ and relay the transcript to $\Vcan$. 

    Second, we show that $\vec a$ passes the canonical verifier's statistical validation $\Validate(\vec q_V,\vec a_V,\tau_V/3,\tau_V)$, according to $\OVcan$. Recall that $\OPcan$ is $\tau_P'$-accurate for $\D_P$ (which is equal to $\D_V$ by assumption) and $\tau_P'\leq \tau_V'$. Similarly, $\OVcan$ is $\tau_V'$-accurate for $\D_V$ and $\tau_V'=\tau_V/3$. Thus, by the specification of $\Validate$ (which uses an argument via triangle inequality), statistical validation passes. Additionally since $\tilde {\vec r}$ is generated via $\Vcan$'s random coins $(r_1,\ldots,r_\ell)$, in all $\Vcan$ will always accept unless $h=\bot$. 

    Thus, it is sufficient to argue that, under the honest protocol, $h\in\GoodV$ with probability at least $1-\gamma$. Recall that $h$ is the output of simulating $(P^{\OPcan},V^{\OPcan})$ on randomness $\tilde{\vec r}$. Since $\OPcan$ is a $\tau_P'$-accurate for $\D_P$ and $\tau_P'\leq\tau_P$, then the simulated $P$ has access to a $\tau_P$-accurate oracle for $\D_P$. Similarly, since $\D_P=\D_V$ by assumption, $\tau_P'\leq \tau_V$, and $\Pcan$ knows all random coins sent by $\Vcan$, then the simulated $V$ has access to a public-state $\tau_V$-accurate oracle for $\D_V$. Since randomness is uniform and generated in epochs according to the canonical verifier, the simulation of $(P^{\OPcan},V^{\OVcan}(\tilde{\vec r}))$ is indistinguishable from a true execution of $(P,V)$. Since $(P,V)$ $(\tau_P,\tau_V,\gamma)$-SQ verifies $\G$, then $h\in\GoodV$ with probability at least $1-\gamma$. Thus, completeness of $(\Pcan,\Vcan)$ is $1-\gamma$.

    \emph{Universal Soundness:}~~  It suffices to show that universal soundness of $\Vcan$ reduces to the universal soundness $V$. Take any cheating prover strategy $\tilde P_\can$ for $\Vcan^{\OVcan}$. 
    We construct a cheating prover $\tilde P$ and oracle $\tilde O_V$ for $V$ as follows. 
    We will argue anytime $\tilde P_\can$ cheat, then $(\tilde P, \tilde O_V)$ cheats as well. 
    Without loss of generality, let $\tilde P_\can$ be deterministic. 
    Let $\tilde P$ and $\tilde \O_V$ each simulate $\tilde P_\can$, where the randomness of each epoch of $V$ is used as input. 
    If the simulation of $\tilde P_\can$ produces a valid transcript $\pi=(\vec q_V, \vec a_V,\tilde{\vec r},\vec s,\vec m)$, then $\tilde \O_V$ responds as the answers $\vec a_V = (a_{1},a_2,\ldots,a_{B_V})$ and $\tilde P$ responds as the messages $\vec m = (m_1,m_2,\ldots)$. 
    To ensure that $\tilde \O_V$ is a $\tau_V$-accurate oracle for $\D_V$, any answer $a$ which is $\tau$-far from the true expectation for $\D_V$ will be replaced with an arbitrary $\tau_V$-accurate answer for $\D_V$.
    In the case where $\tilde P_\can$ does not produce a valid transcript, then $\tilde{P}$ will always reply with a well-formed dummy message and $\tilde \O_V$ will always reply with a $\tau_V$-accurate answer.

    We now reduce from between $(\tilde P_\can,\Vcan^{\OVcan})$ and  $(\tilde P,V^{\tilde \O_V})$. 
    Let the random coins of $\Vcan$ and $V$ be the same random coins. Let us condition on the event that $(\tilde P_\can,\Vcan^{\OVcan})$ accepts. 
    Then $\tilde P_\can$ must have produced a valid  transcript $\pi=(\vec q_V, \vec a_V,\tilde{\vec r},\vec s,\vec m)$ and a hypothesis $h\neq\bot$ where $\pi$ corresponds to a simulation of $V$ that produces $h$. 
    Further, $\Validate(\vec q_V,\vec a_V,\tau_V/3,\tau_V)$ passed. 
    Then $\vec a_V$ are $2\tau_V/3$-close to the answers from $\OVcan(\vec q_V)$. 
    Since $\OVcan$ is $\tau_V/3$-accurate for $\D_V$, then by triangle inequality, $\vec a_V$ are $\tau_V$-accurate for $\D_V$.
    
    Since the random coins are shared between $\Vcan$ and $V$, $(h,\pi)$, as generated above, is also the output of $\tilde P$ and $\tilde \O_V$'s simulation of $\tilde P_\can$.
    Recall that we are conditioning on the event that the $\Vcan$ accepts. Thus, all answers $\vec a_V$ are $\tau_V$-accurate for $\D_V$, and all messages $\vec m$ are well-formed. Thus, $\tilde P$ sends messages $\vec m$, and $\tilde O_V$ provides answers $\vec a_V$. 
    Therefore, the view of the verifier in $(\tilde P, V^{\O_V})$ is indistinguishable from $\Vcan$ simulation of $V$ using $\pi$. Since $\Vcan$ accepted and output $h$, then $V$ must have accepted and also output $h$.

    Suppose that $(\tilde P_\can,\Vcan^{\OVcan})$ accepts and outputs $h\notin\GoodV$ with probability $\gamma'$. Then by monotonicity, $(\tilde P,V^{\tilde \O_V})$ accepts and outputs $h\notin\GoodV$ with probability at least $\gamma'$. But $(\tilde P, V^{\tilde\O_V})$ accepts and outputs $h\notin\GoodV$ with probability at most $\gamma$ by universal soundness of $(P,V)$. Therefore, $\gamma'\leq \gamma$. Since this holds over all $\D_V$ and $\tau_V/3$-accurate oracles $\O_V$ for $\D_V$, then $(\tilde P_\can,\Vcan^{\OVcan})$ accepts and outputs $h\notin\GoodV$ with probability at most $\gamma$. So $(\Pcan,\Vcan)$ satisfies universal soundness $\gamma$.
\end{proof}

\subsection{pvCSVs from the Canonical Protocol and Fiat-Shamir}
\label{sec:FS}
\newcommand{\rnd}{\textsf{rnd}} \newcommand{\sr}{\textsf{SR}} \newcommand{\rom}{\textsf{ROM}}
\newcommand{\SQV}{\textsf{SQV}} \newcommand{\fs}{\mathsf{FS}} \newcommand{\Pfs}{\mathsf{P}_\fs}
\newcommand{\tildePfs}{\tilde{\mathsf{P}}_\fs}
\newcommand{\Vfs}{\mathsf{V}_\fs}
\newcommand{\sfc}{\mathsf{c}} \newcommand{\csv}{\mathsf{csv}} \newcommand{\Pcsv}{\mathsf{P}_\csv}
\newcommand{\Psr}{\mathsf{P}_\sr}
\newcommand{\tildePcsv}{\tilde{\mathsf{P}}_\csv}
\newcommand{\tildePsr}{\tilde{\mathsf{P}}_\sr}
\newcommand{\Vcsv}{\mathsf{V}_\csv}
\newcommand{\gammacsv}{\gamma_\csv}

Next, we construct pvCSVs for SQ computations that use randomness.
Specifically, we compile the canonical Protocol~\ref{prot:canon-epoch} from above into non-interactive SQ protocols using a Fiat-Shamir transformation \cite{fiat1986prove}.
We prove soundness of the transformation in the Random Oracle Model (ROM), which allows us to replace the public coins of the verifier with non-interactive calls to the random oracle.
We begin with preliminaries defining correctness and soundness in the ROM.
Then, we describe our Fiat-Shamir transformation over canonical SQ protocols in Protocol~\ref{prot:fiat-shamir}. 

\paragraph{pvCSVs in the ROM.}

In the Random Oracle Model \cite{bellare1993random}, both the prover and verifier have query access to a shared random function called the random oracle $f:\{0,1\}^*\to\{0,1\}^m$ for some output size $m\in\Nbb$.
Completeness and soundness take probabilities over the sampling of the random oracle which is denoted $f\gets\U$.
Furthermore, we restrict the cheating prover to at most $t\in\Nbb$ queries of the random oracle, called the random oracle budget, typically taken to be polynomial in some security parameter.
We adapt Definition~\ref{def:pvcsv} of pvCSVs to define computationally-sound pvCSVs in the ROM.

\begin{definition}[pvCSV in the ROM]
    \label{def:pvcsv-rom}
    Let $\G$, $\tau_P,\tau_V>0$, $\gammacsv>0$, and $(\Pcsv,\Vcsv)$ be defined in the setup of Definition~\ref{def:pvcsv}. The pvCSV prover $\Pcsv$ and verifier $\Vcsv$ are given query-access to a random oracle $f$. The pvCSV \define{\mathdefine{(\tau_P,\tau_V,\gamma,\gamma')}-certifies \mathdefine{\G} in the ROM} if the following guarantees hold:
    \begin{itemize}
        \item \define{\mathdefine{\gamma}-Completeness in the ROM}:  if $\D_P = \D_V$, then for any verifier distribution $\D_V$, $\tau_V$-accurate oracle $\O_V$ for $\D_V$, $\tau_P$-accurate oracle $\O_P$ for $\D_V$, over sampling of the random oracle $f$, random coins of $\Pcsv$ and $\Vcsv$, the honest prover can generate a hypothesis-certificate pair $(h,\pi) \gets \Pcsv^{\O_P,f}$ such that $\Vcsv^{\O_V,f}$ accepts and $h \in \GoodV$ with probability at least $1-\gamma$. That is, 
        $$
\forall \D_V,\O_V,\O_P.\ \Pr\left[\begin{array}{l}
             \Vcsv^{\O_V,f}(h,\pi)=1  \\
             \land\ h\in\GoodV
        \end{array} \left| 
        \begin{array}{l}
            f\gets\U, \\
            \text{random coins of $\Pcsv$ and $\Vcsv$},\\
            (h,\pi)\gets \Pcsv^{\O_P,f}
        \end{array}\right.\right] \geq 1- \gamma$$
        \item \define{\mathdefine{\gamma'}-Universal Soundness in the ROM}:
        for any verifier distribution $\D_V$ and $\tau_V$-accurate oracle $\O_V$ for $\D_V$, for any random oracle budget $t\in\Nbb$, for any $t$-query prover strategy $\tilde{P}_\csv$,
        for $(h,\pi) \gets \tilde{P}_\csv^{\O_V,f}$, then $V^{\O_V}$ rejects or $h\in\GoodV$ with probability at least $1-\gamma'(t)$
        over sampling of the random oracle $f$, random coins of $\Vcsv$. That is, 
        $$\forall \D_V,\O_V,\tilde{P}_\csv.\ \Pr\left[\begin{array}{l}
             \Vcsv^{\O_V,f}(h,\pi)=0  \\
             \lor\ h\in\GoodV
        \end{array} \left| 
        \begin{array}{l}f\gets\U, \\
            \text{random coins of $V$},\\
            (h,\pi)\gets  \tilde{P}_\csv^{\O_V,f}
        \end{array}\right.\right] \geq 1- \gamma'(t)$$
    \end{itemize}
\end{definition}

We denote ROM soundness against $t$-query provers as $\gamma'(t)$ where $\gamma'$ is a function and $t\in\Nbb$ is the random oracle query budget. Additionally, since soundness is parameterized whereas completeness is not, we split completeness and soundness error into $(\gamma,\gamma')$.

\paragraph{Applying Fiat-Shamir to the Canonical Protocol.}

We transform the Canonical SQ protocol $(\Pcan,\Vcan)$ into a pvCSV in the ROM by applying the Fiat-Shamir transformation to the initial phase of public-coin interaction.
As is standard in Fiat-Shamir, the canonical verifier's randomness in the $i$-th round is the query of the random oracle with the partial transcript of interaction up to round $i$.
We describe the resulting protocol $(\Pfs,\Vfs)$ in Protocol~\ref{prot:fiat-shamir}.
The honest prover $\Pfs$ simulates the interaction $(\Pcan,\Vcan)$ where $\Vcan$'s randomness is sampled via the random oracle $f$.
Then, the verifier $\Vfs$ runs the second and third phase of the canonical protocol as before, checking the correctness of the computational simulation and the statistical validity of the relevant SQs via a batch evaluation of the queries of $\Vcan$.

\begin{algorithm}[t!]
\renewcommand{\algorithmcfname}{Protocol}\caption{Non-Interactive SQ Protocol in Random Oracle Model}
\label{prot:fiat-shamir}
\SetKwInOut{Input}{Input}
\SetKwInOut{Output}{Output}

\textbf{Setup:} Let $(\Pcan,\Vcan)$ be a canonical SQ protocol that $\tau$-SQ verifies $\G$, where
\begin{compactitem}
    \item $\Pfs$ accesses a $\tau$-accurate SQ Oracle $\O_P$ for $\D_P$. \item $\Vfs$ accesses a $\tau$-accurate SQ Oracle $\O_V$ for $\D_V$.
    \item $\ell$ is the round complexity of $(\Pcan,\Vcan)$.
    \item Both $\Pfs$ and $\Vfs$ have access to a random oracle $f$.
\end{compactitem}

\BlankLine
\BlankLine

\textbf{Prover} $\Pfs^{\O_P,f}$:

\Indp
    
    Initialize an empty transcript $\pi_0 = \langle \rangle$
    
    Simulate the interaction $(\Pcan,\Vcan)$ as follows:
    
    \For{$i=1,\ldots,\ell$}{
        $\Pcan^{\O_P}$ sends $\Vcan$ a message $m_i$

        Generate $\Vcan$'s randomness at the $i$-th round as $r_i := f(\langle m_1,\ldots,m_i\rangle )$
        
        $\Vcan$ sends $r_i$ to $\Pcan$     
        
       	Update $\pi_i:=\langle\pi_{i-1},m_i,r_i\rangle$
    }
    Let $\pi:=\pi_{\ell}$
    
    Send $\pi$ to \textbf{Verifier} $\Vfs$.

\Indm

\BlankLine
\BlankLine

\textbf{Verifier} $\Vfs^{\O_V,f}$:

\Indp
	
	Parse the transcript $\pi$ as 
    \begin{compactitem}
        \item $\vec m = (m_1,m_2,\ldots,m_\ell)$ purported messages from $\Pcan$
        \item $\vec r = (r_1,r_2,\ldots,r_\ell)$ purported generated randomness for $\Vcan$
    \end{compactitem}

    \For{$i=1,\ldots,\ell$}{
    		\uIf {$r_i\neq f(\langle m_1,\ldots,m_{i}\rangle)$}
    		{Reject and output $\bot$}
    }

    Execute
    \textbf{Phase 2} and 
    \textbf{Phase 3} of $\Vcan$

    \uIf{either phase fails}{
        Reject and output $\bot$.
    }
    Accept and output $h$.
    
\Indm
 \end{algorithm}

The correctness of the resulting pvCSV follows from the state restoration argument of \cite{ben2016interactive}, with some subtleties that arise in the SQ protocol setting.
The structure of the Canonical protocol simplifies the analysis considerably: the protocol naturally divides into a phase of public-coin interaction, followed by (computational and statistical) verification of the transcript.
Completeness of the protocol follows from completeness of Protocol~\ref{prot:canon-epoch}, since the honest prover can execute the same sequence of queries and computations.
Thus, we focus on establishing soundness.

There are two key aspects of Protocol~\ref{prot:fiat-shamir} that require us to be careful in establishing soundness of the transformation.
First, in our setting, the cheating prover has considerable powers related to the statistical learning problem, with full knowledge of the underlying distribution $\D$ and access to the verifier's SQ oracle.
But, importantly, the cheating prover in the Fiat-Shamir protocol and Canonical protocol are afforded the same powers.
The approach of establishing Fiat-Shamir soundness via state restoration is a black-box reduction, so we can apply the same argument even in our setting where the provers have non-standard computational and statistical powers.
Second, our application of Fiat-Shamir is used to generate a legitimate prover-verifier transcript, which produces a hypothesis $h$, rather than directly certifying a known predicate (e.g., certifying a given $h \in \GoodV$).
That said, once we have a sound transcript, Phase 2 and Phase 3 of the original protocol allow us to validate that the $h$ is actually good for the verifier's distribution.

Specifically, we invoke the following guarantee about the Fiat-Shamir transformation applied to public-coin protocols.

\begin{theorem}[Corollary of Lemma 13.2.7 and Theorem 14.3.1 of \cite{chiesa2024snargsbook}]\label{thm:fs-soundness-upperbound}
    Let $(P,V)$ be a public-coin interactive protocol with round complexity $\ell$ and soundness $\gamma$.
    Let $(\Pfs,\Vfs)$ be the Fiat-Shamir transformed protocol, and suppose it has $\gamma_\fs$ soundness in the ROM.
    There is a black-box reduction that establishes the following upper bound on soundness of $(\Pfs,\Vfs)$ in terms of soundness of $(P,V)$, round complexity $\ell$, and random oracle budget $t\in\Nbb$.
    $$\gamma_\fs(t)\leq \binom{t+\ell}{\ell}\gamma$$
\end{theorem}
Specifically, the reduction goes through a state restoration game and demonstrates how, given a cheating prover for the Fiat-Shamir proof system $\tildePfs$, there exists a cheating prover $\tilde P$ for the original interactive protocol that makes calls to $\tildePfs$, at a $\binom{t+\ell}{\ell}$-factor loss in success probability.
With this fact, we establish the correctness of Protocol~\ref{prot:fiat-shamir}.

\begin{lemma}[Fiat-Shamir for Canonical Protocol]\label{lem:fiat-shamir}
    Fix $\G$, $\tau_P,\tau_V>0$, $\gamma_\can>0$, and canonical SQ protocol $(\Pcan,\Vcan)$ that $(\tau_P,\tau_V,\gamma_\can)$-verifies $\G$ with round complexity $\ell$.
    Then the Fiat-
    Shamir transformed protocol $(\Pfs,\Vfs)$, described in Protocol~\ref{prot:fiat-shamir}, is a pvCSV that $(\tau_P,\tau_V,\gamma_\can,{t + \ell\choose \ell}\gamma_\can)$-certifies $\G$ in the ROM where $t\in\Nbb$ is the random oracle budget.
\end{lemma}

\begin{proof}[Proof sketch]
    Completeness is immediate, by the fact that the random oracle calls are identically distributed to the challenges sent by the public-coin verifier.
    Thus, the honest prover generates a transcript from the same distribution as the honest execution of Protocol~\ref{prot:canon-epoch}.
    
    Soundness follows by Theorem~\ref{thm:fs-soundness-upperbound} applied to Phase 1 of Protocol~\ref{prot:canon-epoch}.
    In particular, the black-box reduction allows us to convert any cheating prover for Protocol~\ref{prot:fiat-shamir} $\tildePfs$ into a cheating prover for Protocol~\ref{prot:canon-epoch} $\tildePcan$.
    While the provers are afforded non-standard SQ oracles and knowledge of the distribution, the Canonical interactive prover $\tildePcan$ has the same oracle access as the Fiat-Shamir prover $\tildePfs$, so $\tildePcan$ can implement the black-box calls to $\tildePfs$ in the reduction.
    Thus, converting the interactive SQ protocol into a non-interactive protocol is sound up to the $\gamma_\fs(t)\leq \binom{t+\ell}{\ell}\gamma_\can$ loss as stated in Theorem~\ref{thm:fs-soundness-upperbound}.
    In particular, the Fiat-Shamir protocol generates a legitimate transcript of the original Canonical protocol, with all but  $\gamma_\fs(t)$ soundness error.

    Finally, we argue that in the SQ model, the verifier's additional checks establish that $h \in \GoodV$ or result in rejection.
    After the non-interactive simulation of the interactive Phase 1,
    the verifier additionally executes Phase 2 and Phase 3 of Protocol~\ref{prot:canon-epoch}.
    Phase 2 certifies that the transcript is a legitimate execution of the SQ protocol (based on the responses to statistical queries) and Phase 3 validates that the query responses are statistically correct. In the SQ model, these checks contribute zero additional soundness error.
    Thus, the upper bound on soundness holds as claimed.
\end{proof}

\paragraph{pvCSVs for all mixed-message protocols with public-verifier state.}
We conclude with a statement of the overall pvCSV guarantee established in this section within the ROM.
In Lemma~\ref{lem:canonical}, we argue that for the class of mixed-message, private-query SQ protocols with public verifier state can be transformed into a Canonical public-coin SQ protocol.
Then, in Lemma~\ref{lem:fiat-shamir}, we apply the Fiat-Shamir transform to turn any Canonical protocol into a pvCSV in the ROM.
Chaining these lemmas together, we obtain the following theorem.

\begin{theorem}\label{thm:pvcsv}
    Fix $\G$, $\tau_P,\tau_V>0$, $\gamma>0$, and a mixed-message, private-query SQ protocol $(P,V)$ with public verifier state. Let $\tau_P'=\min\{\tau_P,\tau_V/3\}$ and $\tau_V'=\tau_V/3$. Suppose $(P,V)$ $(\tau_P,\tau_V,\gamma)$-verifies $\G$. Applying the canonical transformation described in Protocol~\ref{prot:canon-epoch}, and then the Fiat-Shamir transformation, described in Protocol~\ref{prot:fiat-shamir}, yields a pvCSV $(\Pfs,\Vfs)$ that certifies $\G$ with the following characteristics:
    \begin{enumerate}[(a)]
        \item For $V$ with epoch complexity $\ell$, $(\Pfs,\Vfs)$  $(\tau_P',\tau_V',\gamma, {t+\ell\choose\ell}\gamma)$-certifies $\G$ in the ROM.
        \item $\Vfs$ uses a non-adaptive SQ oracle, regardless of the SQ adaptivity in $(P,V)$.
        That is,\\ suppose $V$ has $\tau_V$-query complexity $(k_V,B_V)$ and $P$ has $\tau_P$-query complexity $(k_P,B_P)$; then,
        \begin{itemize}
            \item $\Vfs$ has $\tau_V'$-query complexity $(1,B_V)$;
            \item $\Pfs$ has $\tau_P'$-query complexity $(k_P+k_V,B_P+B_V)$
        \end{itemize}
        \item $\Vfs$ and $\Pfs$ run in linear time in the original protocol.
        That is,\\ suppose $V$ runs in time $T_V(\tau_V)$ and $P$ runs in time $T_P(\tau_P)$; then
        \begin{itemize}
            \item $\Vcan$ runs in time $O(T_V(\tau_V))$
            \item $\Pcan$ runs in time $O(T_V(\tau_V) + T_P(\tau_P))$
        \end{itemize}
        \item For $P$ with communication complexity $c_P$, the certificate size $|\pi|=O(B_V\log(1/\tau_P')) + c_P$. 
\end{enumerate}
\end{theorem}
\begin{proof} Let $(\Pcan,\Vcan)$ be the intermediate, canonical SQ protocol. By Lemma~\ref{lem:canonical}, $(\Pcan,\Vcan)$ $(\tau_P',\tau_V',\gamma)$-verifies $\G$. Since $V$ has epoch complexity $\ell$, $(\Pcan,\Vcan)$ has round complexity $\ell$. Therefore, by Lemma~\ref{lem:fiat-shamir}, $(\Pfs,\Vfs)$ $(\tau_P',\tau_V',\gamma,{t+\ell\choose \ell}\gamma)$-certifies $\G$ in the ROM. So part (a) holds.

Observe that the Fiat-Shamir transformation from $(\Pcan,\Vcan)$ to $(\Pfs,\Vfs)$ preserves many properties including the precision, query complexity, and runtime of the prover and verifier, (at least while assuming unit cost for evaluating the random oracle). Therefore, parts (b) and (c) follow directly from Lemma~\ref{lem:canonical}. 

Let $c_P$ be the communication complexity of $P$. By Lemma~\ref{lem:canonical}, $\Pcan$'s communication complexity is $O(B_V\log(1/\tau_P'))+c_P$. After applying the Fiat-Shamir transformation, the communication complexity is the canonical prover's communication complexity. So, part (d) holds.
\end{proof}

By application of known implementations for adaptive and non-adaptive SQ oracles in Proposition~\ref{prop:full-adaptive-upper} and Proposition~\ref{prop:non-adaptive}, we have the following.

\begin{corollary}
    Let $\G$, $\tau_P$, $\tau_V$, $\gamma$, $(P,V)$, pvCSV $(\Pfs,\Vfs)$ be defined as in the above theorem. Let $\tau_P'=\min\{\tau_P,\tau_V/3\}$. Suppose that both $P$ and $V$ make fully-adaptive queries, that is, $k_P=B_P$ and $k_V=B_V$. Then for all failure probability $\delta>0$, the pvCSV has sample complexity $(m_{\Pfs},m_{\Vfs})$ is as follows:
    \begin{mathpar}
        m_{\Pfs}= O\left(\frac{\sqrt{(k_V+k_P)\cdot\log((k_V+k_P)/\delta)\cdot\log(1/\delta)}}{\tau_P'^2}\right)
        \and 
        m_{\Vfs}=O\left(\frac{\log(k_V/\delta)}{\tau_V^2}\right)
    \end{mathpar}
\end{corollary}

\section{Beyond pvCSVs:  Strengths and Limits of SQ Protocols}
\label{sec:limits}
\newcommand{\del}{\mathsf{del}}
\newcommand{\Pdel}{\mathsf{P}_\del}
\newcommand{\Vdel}{\mathsf{V}_\del}
\newcommand{\OPdel}{\O_{\Pdel}}
\newcommand{\OVdel}{\O_{\Vdel}}

In this section, we consider the power and limitations of SQ protocols.
First, we reiterate that, statistically, SQ protocols are very powerful.
In even more generic settings than our pvCSV constructions, SQ algorithms/protocols can be delegated via interactive SQ protocol such that the verifier's sample complexity scales logarithmically in the number of queries.
Then, we show that SQ protocols---despite their sample efficiency---do not generically provide a computationally efficient verification scheme.
In particular, by a lower bound of \cite{mutreja2023pac} on the sample complexity required to PAC Verify certain VC classes, we show a subexponential lower bound on the query complexity of SQ verification for the same class.

\paragraph{Interactive Non-Adaptive SQ Verification.}

One of the key limitations in our construction of pvCSVs is the reliance on correctness under a public-state SQ oracle.
We show that it is possible to achieve non-adaptive statistical verification of SQ algorithms that are only correct under an oblivious SQ oracle, albeit with interaction.
This result is analogous to a result showed in the recent journal version of \cite{mutreja2023pac}.
Given any SQ algorithm, the verifier simply executes the algorithm using the prover as its oracle.
The prover responds interactively to each query.
And then at the end, the verifier checks the answers of the prover non-adaptively.

In fact, this simple idea also establishes that a large class of SQ protocols---even more general than those covered by Lemma~\ref{lem:canonical}---can be verified using non-adaptive statistical complexity.
Specifically, for any public-query (private-coin) SQ protocol that verifies $\G$ with public verifier state, there is an implementation of the protocol that only requires non-adaptive verifier sample complexity.

\begin{proposition}
    \label{prop:public-query}
    Fix a learning goal $\G$, $\tau_P,\tau_V > 0$, and $\gamma > 0$, and let $\tau_P' = \min\{\tau_P,\tau_V/3\}$ and $\tau_V' = \tau_V/3$.
    Suppose $(P,V)$ is a public-query SQ protocol that $(\tau_P,\tau_V,\gamma)$-verifies $\G$ with public verifier state, where $P$ has $\tau_P$-query complexity $(k_P,B_P)$ and $V$ has $\tau_V$-query complexity $(k_V,B_V)$.
    There is an SQ protocol $(P',V')$ that $(\tau_P',\tau_V',\gamma)$-verifies $\G$ with public verifier state, where $P'$ has $\tau_P'$-query complexity $(k_P + k_B,B_P + B_V)$ and $V'$ has non-adaptive $\tau_V'$-query complexity $(1,B_V)$.
\end{proposition}

Proposition~\ref{prop:public-query} follows, again, by asking the prover to make the verifier's statistical queries.
The completeness, soundness, and resulting bounds are analogous to those established by Lemma~\ref{lem:canonical}.

\paragraph{Query Lower Bound for SQ Verification of a VC Class.}
We show that the statistical upper bound we achieve for the SQ verifier actually implies a computational lower bound for SQ protocols.
Specifically, there exists a hypothesis class of VC dimension $d$ that cannot be SQ verified, even using the most general interactive SQ protocols, using polynomially many queries in $d$.
Our lower bound piggybacks off of the lower bound proved in \cite{mutreja2023pac} for PAC Verification.\footnote{Informally, PAC Verification is the problem of delegation of learning, for the specific learning goal of Agnostic PAC learning.
We refer the unfamiliar reader to \cite{goldwasser2021interactive}.}

\begin{theorem*}[Restatement of Theorem~2.1 of \cite{mutreja2023pac}]
Fix $\eps > 0$, $\delta = 1/3$, and let $(P,V)$ be an interactive proof system for learning.
For any hypothesis class $\H$ of VC Dimension $d$, if $(P,V)$ PAC verifies $\H$ (for all distributions $\D$) with accuracy $\eps$ and failure probability $\delta$, then the verifier $V$ must use at least $m_V \ge \Omega(\sqrt{d}/\eps^2)$ i.i.d.\ samples from $\D$.
\end{theorem*}

Stringing this lower bound on the sample complexity of PAC verifying a VC class with our upper bound on the sample complexity of SQ verification, we immediately get a query complexity lower bound for SQ verification.
Taking our upper bound of $\log(k)/\tau^2$ from Proposition~\ref{prop:public-query}, in terms of the number of statistical queries $k$, and their lower bound of $\sqrt{d}/\eps^2$ in terms of the VC dimension $d$, we obtain the following corollary.
\begin{corollary}
    \label{cor:LB}
    Fix $\gamma = 1/3$ and fix the verifier tolerance $\tau_V = \Omega(\eps) > 0$ in terms of the agnostic learning accuracy $\eps$.
    For every hypothesis class $\H$ of VC dimension $d$, and for any $\tau_P > 0$, any public-query SQ protocol that $(\tau_P,\tau_V,\gamma)$-SQ verifies $\eps$-Agnostic PAC learning of $\H$ with public verifier state requires the verifier to make $k = 2^{\Omega(\sqrt{d})}$ statistical queries.
\end{corollary}
That is, even for the strongest model of SQ protocol we consider---interactive, public-query SQ protocols that use a public-state SQ oracle---when the tolerance of the SQ oracle is fixed $\tau_V \approx \eps$, verification may be computationally inefficient, despite statistical efficiency.

\section{Differential Privacy and SQ Verification}

\label{sec:dp-verification}

In many statistical analyses, Differential Privacy (DP) \cite{dwork2006calibrating} is a desirable property to satisfy to maintain the privacy of individuals' data within a database.\footnote{DP is the gold standard notion for privacy-protections in statistical analyses.  We refer the unfamiliar reader to introductory materials on DP, including \cite{DworkRoth2014,Vadhan2017,Kamath2020}.  For the sake of presentation, we omit background on the Exponential Mechanism of \cite{McSherryTalwar2007}.}
DP protects individuals' privacy by requiring stability in a randomized algorithm's behavior on neighboring databases $D,D'$ that differ on a single element.
\begin{definition}[Differential Privacy \cite{dwork2006calibrating}]
    Fix a domain $\X$ and range $\R$, and $\eps,\delta > 0$.
    A randomized algorithm $M:\X^* \to \R$ is \define{\mathdefine{(\eps,\delta)}-differentially private} if for all neighboring databases $D,D' \in \X^*$ and for all measurable subsets $S \subseteq \R$,
\[
  \Pr[M(D) \in S] \le e^{\eps} \cdot \Pr[M(D') \in S] + \delta.
\]
\end{definition}
DP provides rigorous protections to individuals, but comes at a cost.
Even for non-adaptive statistical queries, in high-dimensional settings ($d \ge k$), answering $k$ queries requires sample complexity scaling polynomially in $k$.
Concretely, to answer $k$ queries with $\tau$ tolerance under pure $(\eps,0)$-DP requires $\Omega(k/\tau\eps)$ i.i.d.\ samples from $\D$ \cite{HardtTalwar2010}; under approximate $(\eps,\delta)$-DP requires $\Omega(\sqrt{k \cdot \log(1/\delta)}/\tau\eps)$ \cite{SteinkeUllman2016}.
Both of these results are tight (up to poly-logarithmic factors); in fact, the upper bound on answering adaptive statistical queries is tightly connected to the upper bound on answering queries under approximate DP.

Certifying the results of a statistical analysis, while maintaining DP with respect to the verifier's samples, offers a potential for significant savings.
We show, generically, that the verifiers for our pvCSVs (in fact, for all of our SQ protocols) can be implemented under DP using essentially the same non-adaptive sample complexity.

\begin{algorithm}[t]
\label{alg:dpvalidate}
\caption{$\mathsf{DPValidate}(\vec v, \vec a, \tau_P, \tau)$~~ i.e., Differentially-Private Statistical Validation}
\SetKwInOut{Input}{Input}
\SetKwInOut{Setup}{Setup}

\Input{Batch of statistical queries $\vec q$;\\ List of candidate answers to queries $\vec a$;\\
Candidate accuracy $\tau_P > 0$ and final accuracy $\tau > 0$, where $\tau_P \le \tau/3$}
\Setup{$m_V$ i.i.d.\ samples drawn from $\D_V$; $x_1,\hdots,x_{m_V}$\\
Sensitiviy of empirical queries $\Delta = 1/m_V$}

\BlankLine
\BlankLine

\For{$j = 1,\hdots,B$}{
    Define $\hat{a}_j = \frac{1}{m_V} \sum_{i=1}^{m_V} q_j(x_i)$
    
    Define $\nu_j = |\hat{a}_j - a_j|$
}
Use the Exponential Mechanism to sample a noisy maximum $\nu^*$ over $\{\nu_j\}$ according to:
\begin{gather*}
    \Pr[\nu^* = \nu_j] \propto \exp\left(\frac{\eps \cdot \nu_j}{2 \Delta }\right)
\end{gather*}

\uIf{$\nu^* > \tau/2$}
{\Return False}
\Return True
\end{algorithm}

\begin{proposition}
    \label{prop:dp-verifier}
There exists an $(\eps,0)$-DP implementation of the
    non-adaptive statistical validation, given in Algorithm~\ref{alg:dpvalidate},
with failure probability $\beta$ that uses $m_V$ i.i.d.\ samples from $\D_V$, where for any number of adaptive queries $k$,
    $$m_V \le O\left(\frac{\log(k/\beta)}{\tau^2} + \frac{\log(k/\beta)}{\tau\eps}\right).$$
\end{proposition}
\begin{proof}[Proof sketch]
The proposition follows from a standard application of the Exponential Mechanism \cite{McSherryTalwar2007}.
The original non-adaptive statistical validation step evaluates each query on the samples from $\D_V$.\footnote{This is the simplest implementation of Algorithm~\ref{alg:validate} given i.i.d.\ samples.}
Namely, the original verifier computes $\hat{a}_1,\hdots,\hat{a}_k$ where $\hat{a}_j = \frac{1}{m_V}\sum_{i=1}^{m_V} q_j(x_i)$.
Instead, the new verifier checks the (noisy) maximum difference between the statistics reported by the prover and the empirical statistics on their samples, $|a_j - \hat{a}_j|$.

To compute the new verifier's sample complexity, we need to reason about the accuracy of the empirical statistics $\hat{a}_1,\hdots,\hat{a}_k$, as well as the accuracy of the exponential mechanism.
For target tolerance $\tau$ and prover tolerance $\tau/3$ (as in our protocols), we will insist that each of these components achieves $\tau/16$ additive error with all but $\beta/2$ probability each.

First, the accuracy of the empirical statistics:
per Proposition~\ref{prop:non-adaptive}, $m_V$ can scale as $O(\log(k/\beta)/\tau^2)$.
Then, the accuracy of the release of the maximum difference $\max_{1\le j \le k}|a_j - \hat{a}_j|$:
per \cite{McSherryTalwar2007}, with all but $\beta/2$ probability, the exponential mechanism achieves error $\alpha$ where
\begin{gather*}
    \alpha \le \frac{2\Delta}{\eps} \cdot \log(2k/\beta)
\end{gather*}
where $\Delta = 1/m_V$ is the sensitivity of the release from a database of $m_V$ samples.
Thus, the exponential mechanism guarantees $\tau/16$-accuracy to the empirical statistics for some $m_V \le O(\log(k/\beta)/\tau\eps)$.
With these accuracies fixed with all but $\beta$ total failure probability, we can define the new verifier and analyze Completeness and Soundness.

For a sequence of $k$ fully-adaptive SQs, Algorithm~\ref{alg:dpvalidate} runs the exponential mechanism to release the noisy maximum difference $\max_{1\le j \le k} |a_j - \hat{a}_j|$.
If the difference is less than $\tau/2$, it accepts; otherwise, it rejects.

\emph{Completeness:}~~Per our earlier protocols, we assume the honest prover reports each $a_j$ within $\tau_P = \tau/3$ of the true query value.
The verifier's empirical statistics $\hat{a}_j$ are within $\tau/16$ of the true query value, and their difference is released with $\tau/16$-accuracy, so in sum, the reported maximum difference will be strictly less than $\tau/2$, with all but $\beta$ probability.

\emph{Soundness:}~~Suppose there is some statistic $a_j$ that was reported with more than $\tau$ additive error from the true value, so the SQ guarantee is violated.
Again, the verifier's accuracy ensures that the difference $|a_j - \hat{a}_j|$ is at least $15\tau/16$, which is released with at most $\tau/16$ additive error.
So, with all but $\beta$ probability, the reported maximum difference is strictly more than $\tau/2$.
\end{proof}

\clearpage
\section*{Acknowledgements}
The authors thank Noah Stephens-Davidowitz for significant conversations at the start of this work and feedback throughout the project.
We also thank Robert Kleinberg, Jonathan Shafer, and Nick Spooner for helpful discussions.

\bibliographystyle{alpha}
\bibliography{refs}

\end{document}